\newcommand{\shortversion}[1]{}
\newcommand{\longversion}[1]{#1}

\def\year{2018}\relax
\documentclass[letterpaper]{article} 
\usepackage{aaai18}  
\usepackage{times}  
\usepackage{helvet}  
\usepackage{courier}  
\usepackage{url}  
\usepackage{graphicx}  
\usepackage{bm}
\usepackage{subcaption}

\usepackage{balance}

\usepackage{url}
\usepackage[usenames,dvipsnames]{color}
\usepackage{todonotes}

\usepackage{microtype}

\usepackage{booktabs}

\DeclareRobustCommand{\DE}[3]{#2}

\usepackage{amsthm}
\usepackage{amssymb}
\newtheorem{definition}{Definition}
\newtheorem{theorem}{Theorem}
\newtheorem{theorem*}[theorem]{Theorem$^{\star}$}
\newtheorem{lemma}[theorem]{Lemma}
\newtheorem{lemma*}[theorem]{Lemma$^{\star}$}
\newtheorem{proposition}[theorem]{Proposition}
\newtheorem{proposition*}[theorem]{Proposition$^{\star}$}

\newtheorem{corollary}[theorem]{Corollary}
\newtheorem{corollary*}[theorem]{Corollary$^{\star}$}

\newtheorem{example}{Example}

\usepackage{mathtools}

\DeclareMathOperator*{\argmax}{arg\,max}
\DeclareMathOperator*{\argmin}{arg\,min}

\newcommand{\SB}{\{\,}%
\newcommand{\SM}{\;{:}\;}%
\newcommand{\SE}{\,\}}%
\newcommand{\SBs}{\{}%
\newcommand{\SEs}{\}}%

\renewcommand{\P}{\text{\normalfont P}}
\newcommand{\NP}{\text{\normalfont NP}}

\newcommand{\co}{\text{\normalfont co-}}

\newcommand{\NN}{\mathbb{N}}
\newcommand{\ZZ}{\mathbb{Z}}

\newcommand{\AAA}{\mathcal{A}}

\newcommand{\CCC}{\mathcal{C}}

\newcommand{\III}{\mathcal{I}}

 \newcommand{\RRR}{\mathcal{R}}

\newcommand{\Card}[1]{|#1|}

\newcommand{\mtext}[1]{\text{\normalfont #1}}

\newcommand{\CNF}{\ensuremath{\mtext{\sc CNF}}}

\newcommand{\Horn}{\ensuremath{\mtext{\sc Horn}}}
\newcommand{\Krom}{\ensuremath{\mtext{\sc Krom}}}
\newcommand{\DefHorn}{\ensuremath{\mtext{\sc DefHorn}}}

\newcommand{\MinVC}[0]{\ensuremath{\mtext{\textsc{Min}}\-\mtext{\textsc{Vertex}}\-\mtext{\textsc{Cover}}}}

\newcommand{\Var}[1]{\mtext{Var\ensuremath{(#1)}}}
\newcommand{\Lit}[1]{\mtext{Lit\ensuremath{(#1)}}}

\newcommand{\ThetaP}[1]{\ensuremath{\Theta^{\mtext{p}}_{#1}}}
\newcommand{\DeltaP}[1]{\ensuremath{\Delta^{\mtext{p}}_{#1}}}

\newcommand{\prof}[1]{\text{\boldmath $#1$}}

\newcommand{\Outcome}[1]{\ensuremath{\mtext{\textsc{Out}}\-\mtext{\textsc{come}}\-\mtext{\textsc{(#1)}}}}

\newcommand{\Majority}[0]{\mtext{\textsc{maj}}}
\newcommand{\Kemeny}[0]{\mtext{\textsc{med}}}
\newcommand{\Slater}[0]{\mtext{\textsc{mcc}}}
\newcommand{\Young}[0]{\mtext{\textsc{young}}}
\newcommand{\Tideman}[0]{\mtext{\textsc{ra}}}
\newcommand{\ReversalScoring}[0]{\mtext{\textsc{rev}}}

\newcommand{\ms}[0]{\mtext{ms}}
\newcommand{\MaxHamming}[0]{\mtext{\textsc{maxham}}}

\newcommand{\MaxModel}[0]{\ensuremath{\mtext{\textsc{Max}-}\allowbreak{}\mtext{\textsc{Model}}}}

\newenvironment{myquote}{\begin{center}
    \begin{minipage}{.98\linewidth}}{\end{minipage}\end{center}}
    
\newcommand{\probdef}[1]{
  \begin{myquote}
    \framebox[\linewidth][l]{\parbox{\linewidth-10pt}{
    #1
    }}
  \end{myquote}
}

\newcommand{\citealt}[1]{\citeauthor{#1} \citeyear{#1}}
\newcommand{\citealp}[1]{\citeauthor{#1} (\citeyear{#1})}

\begin{document}
%
\title{Hunting for Tractable Languages for Judgment Aggregation}
\author{Ronald de Haan\\
Institute for Logic, Language and Computation\\
University of Amsterdam\\
\url{me@ronalddehaan.eu}
}
\maketitle

\begin{abstract}
Judgment aggregation is a general framework for collective
decision making that can be used to model many different
settings.
Due to its general nature, the worst case complexity of
essentially all relevant problems in this framework is very high.
However, these intractability results are mainly due to the
fact that the language to represent the aggregation domain
is overly expressive.
We initiate an investigation of representation languages for
judgment aggregation that strike a balance between
(1)~being limited enough to yield computational tractability results
and (2)~being expressive enough to model relevant applications.
In particular, we consider the languages of Krom formulas,
(definite) Horn formulas, and Boolean circuits in decomposable
negation normal form (DNNF).
We illustrate the use of the positive complexity results
that we obtain for these languages with a concrete application:
voting on how to spend a budget
(i.e., participatory budgeting).
\end{abstract}

\section{Introduction}

Judgment aggregation is a general framework to study methods
for collective opinion forming, that has been investigated in the area
of computational social choice
(see, e.g.,~\citealt{Endriss16}, \citealt{GrossiPigozzi14}).
The framework is set up in such a general way
that it can be used to model an extremely
wide range of scenarios---including,
e.g., the setting of voting \cite{DietrichList07b}.
On the one hand, this generality is an advantage: methods
studied in judgment aggregation can be employed in all
these scenarios.
On the other hand, however, this generality severely hinders
the use of judgment aggregation methods in applications.
Because there are no restrictions on the type of aggregation
settings that are modeled,
relevant computational tasks across the board
are computationally intractable in the worst case.
In other words, no performance guarantees are available
that warrant the efficient use of judgment aggregation methods
for applications---not even for simple settings.
For example, computing the outcome of a judgment aggregation
scenario is NP-hard for all aggregation procedures studied in the
literature that satisfy the rudimentary quality condition
of \emph{consistency}
\cite{EndrissDeHaan15,DeHaanSlavkovik17,LangSlavkovik14}.%

These negative computational complexity results are in many cases
due purely to the expressivity of the language used to represent
aggregation scenarios (full propositional logic, or CNF formulas)---%
not to the structure of the scenario being modeled.
In other words, the known negative complexity results draw
an overly negative picture

To correct this gloomy and misleading image, a more detailed and
more fine-grained perspective is needed on the way
that application settings are modeled in the general framework
of judgment aggregation.
In this paper, we take a first look at the complexity of judgment aggregation
scenarios using this more sensitive point of view.
That is, we initiate an investigation of representation languages for
judgment aggregation that (1)~are modest enough to
yield positive complexity results for relevant computational tasks,
yet (2)~are general enough to model interesting and
relevant applications.

Concretely, we look at several restricted propositional languages
that strike a balance between expressivity and tractability in other
settings, and we study to what extent such a balance is attained
in the setting of judgment aggregation.
In particular, we look at Krom (2CNF), Horn and definite Horn
formulas, and we consider the class of Boolean circuits in
decomposable negation normal form (DNNF).
We study the impact of these restricted languages on the
complexity of computing outcomes for a number of
judgment aggregation procedures studied in the literature.
We obtain a wide range of (positive and negative) results.
Most of the results we obtain are summarized in
Tables~\ref{table:krom-horn-results},
\ref{table:dnnf-results}
and~\ref{table:budget-results},
located in later sections.

In particular, we obtain several interesting
positive complexity results for the case where
the domain is represented using a Boolean circuit
in DNNF.
Additionally, we illustrate how this representation language
of Boolean circuits in DNNF---%
that combines expressivity and tractability---%
can be used to get tractability results for a specific
application: 
voting on how to spend a budget.
This application setting can be seen as an instantiation of the
setting of \emph{Participatory Budgeting}
(see, e.g.,~\citealt{BenadeNathProcacciaShah17}).

\paragraph{Related Work}
Judgment aggregation has been studied
in the field of computational social choice
from (a.o.) a philosophy,
economics and computer science perspective
(see, e.g., \citealt{Dietrich07}, \citealt{Endriss16},
\citealt{GrossiPigozzi14},
\citealt{LangPigozziSlavkovikVanderTorreVesic17},
\citealt{ListPettit02},
\citealt{Rothe16}).
The complexity of computing outcomes for
judgment aggregation procedures
has been studied by, a.o.,~\citealp{EndrissGrandiPorello12},
\citealp{EndrissGrandiDeHaanLang16},
\citealp{EndrissDeHaan15},
\citealp{DeHaanSlavkovik17} and
\citealp{LangSlavkovik14}.
See Table~\ref{table:bg-results} for
complexity results that are relevant for this paper.

\paragraph{Roadmap}
We begin by explaining the framework of judgment aggregation.
We then study to what extent the known languages of
Krom and (definite) Horn formulas lead to suitable results for
judgment aggregation.
We continue with looking at the class of DNNF circuits---%
studied in the field of knowledge compilation---and we illustrate
how results for this class of circuits can be used for a
concrete application of judgment aggregation
(that of voting on how to allocate a budget).
We conclude with outlining some promising ways in which
the research path that we set out can be followed.

An overview of notions
from propositional logic and computational complexity theory
that we use can be found in the appendix.
\longversion{%
The proofs of some results are omitted from the main paper
and are located in the additional material at the end%
---these results are marked with a star ($\star$).}
\shortversion{%
The proofs of some results are omitted from the paper%
---these results are marked with a star ($\star$).
Full proofs can be found in the accompanying
technical report \textcolor{red}{\cite{DeHaan18}}.}

\section{Judgment Aggregation}

We begin by introducing the setting of Judgment Aggregation
\cite{Dietrich07,Endriss16,GrossiPigozzi14,ListPettit02}.
In this paper, we will use a variant of the framework that has been studied
by, e.g.,~\citealp{Grandi12}, \citealp{GrandiEndriss13}
and \citealp{EndrissGrandiDeHaanLang16}.%
\footnote{This framework is also known under the name
of \emph{binary aggregation with integrity constraints},
and can be used interchangeably with other Judgment Aggregation
frameworks from the literature
---%
as shown by~\citealp{EndrissGrandiDeHaanLang16}.}

Let~$\III = \SBs x_1,\dotsc,x_n \SEs$ be a finite set of \emph{issues},
in the form of propositional variables.
Intuitively, these issues are the topics about which the individuals want
to combine their judgments.
A truth assignment~$\alpha : \III \rightarrow \SBs 0,1 \SEs$ is
called a \emph{ballot}, and represents
an opinion that individuals and the group can have.
%
%
We will also denote ballots~$\alpha$ by a binary
vector~$(b_1,\dotsc,b_{n}) \in \SBs 0,1 \SEs^{n}$,
where~$b_i = \alpha(x_i)$ for each~$i \in [n]$---%
we use~$[n]$ to denote~$\SBs 1,\dotsc,n \SEs$ for each~$n \in \NN$.
Moreover, we say that~$(p_1,\dotsc,p_{n}) \in \SBs 0,1,\star \SEs^{n}$
is a \emph{partial ballot}, and that~$(p_1,\dotsc,p_{n})$ \emph{agrees
with} a ballot~$(b_1,\dotsc,b_{n})$ if~$p_i = b_i$ whenever~$p_i \neq \star$,
for all~$i \in [n]$.
%
We use an integrity constraint~$\Gamma$ to restrict the
set of feasible opinions (for both the individuals and the group).
The integrity constraint~$\Gamma$ is a propositional formula
(or more generally, a single-output Boolean circuit),
whose variables can include~$x_1,\dotsc,x_n$.
We define the set~$\RRR(\III,\Gamma)$ of \emph{rational ballots}
to be the ballots (for~$\III$) that are consistent with the integrity
constraint~$\Gamma$.
We say that finite sequences~$\prof{r} \in \RRR(\III,\Gamma)^{+}$
of rational ballots are \emph{profiles}.
A profile contains a ballot for each individual participating
in the judgment aggregation scenario.
Where convenient we equate a profile~$\prof{r} = (r_1,\dotsc,r_p)$
with the multiset containing~$r_1,\dotsc,r_p$.
%

A \emph{judgment aggregation procedure} (or \emph{rule}),
for the set~$\III$ of issues and the integrity constraint~$\Gamma$,
is a function~$F$ that takes as input a
profile~$\prof{r} \in \RRR(\III,\Gamma)^{+}$,
and that produces a non-empty set of ballots.
A procedure~$F$ is called \emph{consistent} if
for all~$\III$,~$\Gamma$ and~$\prof{r}$
it holds that each~$r^{*} \in F(\prof{r})$
is consistent with~$\Gamma$.
Consistency is a central requirement for judgment aggregation
procedures, and all rules that we consider in this paper
are consistent.

An example of a simple judgment aggregation procedure
is the \emph{majority rule} (defined for profiles with an
odd number of ballots).
We let the majority outcome~$m_{\prof{r}}$ be the partial ballot
such that for each~$x \in \III$,~$m_{\prof{r}}(x) = 1$
if a strict majority of ballots~$r_i \in \prof{r}$ satisfy~$r_i(x) = 1$,%
~$m_{\prof{r}}(x) = 0$
if a strict majority of ballots~$r_i \in \prof{r}$ satisfy~$r_i(x) = 0$,
and~$m_{\prof{r}}(x) = \star$ otherwise.
The majority rule returns the majority outcome~$m_{\prof{r}}$.
The majority rule is efficient to compute, but is not consistent
(as shown in Example~\ref{ex:ja-setting}).

\begin{example}
\label{ex:ja-setting}
Consider the judgment aggregation scenario.
where~$\III = \SBs x_1,x_2,x_3 \SEs$,
$\Gamma = (\neg x_1 \vee \neg x_2 \vee \neg x_3)$,
and the profile~$\prof{r} = (r_1,r_2,r_3)$
is as shown in Table~\ref{table:ja-example}.
The majority outcome~$\Majority(\prof{r})$
is inconsistent with~$\Gamma$.
\begin{table}[h!]
\begin{center}
  \begin{tabular}{c || c c c}
    \toprule
    $\prof{r}$ & $x_1$ & $x_2$ & $x_3$ \\
    \midrule
    $r_1$ & $1$ & $1$ & $0$ \\
    $r_2$ & $1$ & $0$ & $1$ \\
    $r_3$ & $0$ & $1$ & $1$ \\
    \midrule
    $\Majority(\prof{r})$ & $1$ & $1$ & $1$ \\
    \bottomrule
  \end{tabular}
\end{center}
  \vspace{-5pt}
  \caption{Example of a judgment aggregation scenario.}
  \label{table:ja-example}
  \vspace{-10pt}
\end{table}
\end{example}

\subsection{Judgment Aggregation Procedures}

Next, we introduce the judgment aggregation rules
that we use in this paper.
These procedures are all consistent
and are many of the ones that have been
studied in the literature
(for an overview see,
e.g.,~\citealt{LangPigozziSlavkovikVanderTorreVesic17}).


Several procedures that we consider can be seen as
instantiations of a general template: \emph{scoring procedures}.
Let~$\III$ be a set of issues and~$\Gamma$ be an integrity constraint.
Moreover, let~$s : \RRR(\III,\Gamma) \times \Lit{\III} \rightarrow \NN$ be
a \emph{scoring function} that assigns a value to each literal~$l \in \Lit{\III}$
with respect to a ballot~$r \in \RRR(\III,\Gamma)$.
The scoring judgment aggregation procedure~$F_s$ that corresponds
to~$s$ is defined as follows:
\[ F_s(\prof{r}) =
  \argmax\limits_{r \in \RRR(\III,\Gamma)}
  \sum\limits_{r_i \in \prof{r}}
  \sum\limits_{l \in \Lit{\III} \atop r(l) = 1}
  s(r_i,l).
\]
That is,~$F_s$ selects the rational ballots~$r \in \RRR(\III,\Gamma)$
that maximize the cumulative score for all literals agreeing with~$r$
with respect to all ballots~$r_i \in \prof{r}$.

The \emph{median (or Kemeny) procedure}~\Kemeny{}
is based on the scoring function and is defined
by letting~$s_{\mtext{K}}(r,l) = r(l)$
for each~$r \in \RRR(\III,\Gamma)$ and
each~$l \in \Lit{\III}$.
Alternatively, the \Kemeny{} procedure can be defined as the rule
that selects the ballots~$r^{*} \in \RRR(\III,\Gamma)$
that minimize the cumulative Hamming distance to the profile~$\prof{r}$.
The \emph{Hamming distance}
between two ballots~$r,r'$
is~$d_{\mtext{H}}(r,r') = \Card{\SB x \in \III \SM r(x) \neq r'(x) \SE}$.

The \emph{reversal scoring procedure}~\ReversalScoring{}
is based on the scoring
function~$s_{\mtext{R}}(r,l)$ such
that~$s_{\mtext{R}}(r,l) = \min\nolimits_{r' \in \RRR(\III,\Gamma), r'(l) = 0}
d_{\mtext{H}}(r,r')$
for each~$r \in \RRR(\III,\Gamma)$ and each~$l \in \Lit{\III}$.
That is, the score~$s_{\mtext{R}}(r,l)$ of~$l$ w.r.t.~$r$
is the minimal number of issues whose truth value needs to be flipped
to get a rational ballot~$r'$ that sets~$l$ to false.


The \emph{max-card Condorcet (or Slater) procedure}~\Slater{}
is also based on the Hamming distance.
Let~$\prof{r}$ be a profile.
The \Slater{} procedure is defined
by letting~$\Slater(\prof{r}) = \argmin\nolimits_{r^{*} \in \RRR(\III,\Gamma)}
d_{\mtext{H}}(r^{*},m_{\prof{r}})$.
That is, the \Slater{} procedure selects the rational ballots
that minimize the Hamming distance
to the majority outcome~$m_{\prof{r}}$.




The \emph{Young procedure}~\Young{}
selects those ballots that
can be obtained as a rational majority outcome by deleting a minimal number
of ballots from the profile.
Let~$\prof{r}$ be a profile, and let~$d$ denote the smallest number such that
deleting~$d$ individual ballots from~$\prof{r}$ results in a profile~$\prof{r}'$
such that~$m_{\prof{r}'}$ is a complete and rational ballot.
We let the outcome~$\Young(\prof{r})$ of the Young procedure
be the set of rational ballots~$r^*$ such that deleting~$d$ individual
from~$\prof{r}$ results in a profile~$\prof{r}'$ with~$m_{\prof{r}'} = r^*$.


The \emph{Max-Hamming procedure}
\MaxHamming{}
is also based on the Hamming distance.
Let~$r$ be a single ballot,
and let~$\prof{r} = (r_1,\dotsc,r_p)$ be a profile.
We define the max-Hamming distance between~$r$ and~$\prof{r}$
to be~$d_{\mtext{max,H}}(r,\prof{r}) = \max\nolimits_{r_i \in \prof{r}} d_{\mtext{H}}(r,r_i)$.
The Max-Hamming procedure is defined
by letting~$\MaxHamming(\prof{r}) =
\argmin\nolimits_{r^{*} \in \RRR(\III,\Gamma)}
d_{\mtext{max,H}}(r^{*},\prof{r})$.
That is, the Max-Hamming procedure selects the rational ballots
that minimize the max-Hamming distance to~$\prof{r}$.


The \emph{ranked agenda (or Tideman) procedure}
\Tideman{}
is based on the notion of majority strength.%
\footnote{Here, we consider a variant of the ranked agenda
procedure that works with a fixed tie-breaking order.
Other variants, where all possible tie-breaking orders are considered
in parallel, have also been studied in the literature
(see, e.g.,~%
\citealt{LangPigozziSlavkovikVanderTorreVesic17}).}
Let~$\prof{r}$ be a profile and let~$l \in \Lit{\III}$.
The majority strength~$\ms(\prof{r},l)$ of~$l$ for~$\prof{r}$
is the number of ballots~$r \in \prof{r}$ such that~$r(l) = 1$.
Let~$<_{\mtext{tb}}$ be a fixed linear order on~$\Lit{\III}$
(the tie-breaking order).
Based on~$<_{\mtext{tb}}$ and the majority strength,
we define the linear order~$<_{\prof{r}}$ on~$\Lit{\III}$.
Let~$l_1,l_2 \in \Lit{\III}$.
Then~$l_1 <_{\prof{r}} l_2$ if either (i)~$\ms(\prof{r},l_1) > \ms(\prof{r},l_2)$
or (ii)~$\ms(\prof{r},l_1) = \ms(\prof{r},l_2)$ and~$l_1 <_{\mtext{tb}} l_2$.
Then~$\Tideman(\prof{r}) = \SBs r^{*} \SEs$ where the ballot~$r^{*}$ is
defined inductively as follows.
Let~$l_1, l_2, \dotsc, l_{2n}$ be such
that for each~$i \in [2n-1]$ it holds that~$l_i <_{\prof{r}} l_{i+1}$.
Let~$s_0$ be the empty truth assignment.
For each~$i \in [2n-1]$, check whether both~$s_i(l_i) \neq 0$
and~$s'_i$ is consistent with~$\Gamma$,
where~$s'_i$ is obtained from~$s_i$ by setting~$l_i$ to true
(and keeping the assignments to variables not occurring in~$l_i$ unchanged).
If both are the case, then let~$s_{i+1} = s'_i$.
Otherwise, let~$s_{i+1} = s_i$.
Then~$r^{*} = s_{2n}$.
Intuitively, the procedure iterates over
the assignments~$l_1,l_2,\dotsc$ in the order specified by~$<_{\prof{r}}$.
Each literal~$l_i$ is set to true whenever this does
not lead to an inconsistency with previously assigned literals.

\subsection{Outcome Determination}

When given a judgment aggregation scenario (i.e., an agenda, an integrity
constraint, and a profile of individual opinions),
an important computational task is to compute a
possible collective opinion, for a fixed judgment aggregation procedure.
This task is often referred to as \emph{outcome determination}.
Moreover, often it makes sense to seek possible collective opinions that
satisfy certain properties (e.g., whether or not a given issue is accepted
in the collective opinion).

Essentially, this is a search problem: the task is to find one of (possibly)
multiple solutions.
However, to make the theoretical complexity analysis easier, we will
consider the following decision variant of this problem.

\probdef{
  \Outcome{F}
  
  \emph{Instance:} A set~$\III$ of issues with
    an integrity constraint~$\Gamma$
    a profile~$\prof{r} \in \RRR(\III,\Gamma)^{+}$
    and a partial ballot~$s$ (for~$\III$).
  
  \emph{Question:} Is there a
    ballot~$r^* \in F(\prof{r})$
    such that~$s$ agrees with~$r^*$?
}

An outcome~$r^{*}$ witnessing a yes-answer can be obtained
by solving this decision problem a linear number of times.
In addition to the basic task of finding one outcome
(that agrees with a given partial ballot~$s$), one could consider
other computational tasks, e.g., representing the set~$F(\prof{r})$
of outcomes in a succinct way that admits certain queries/operations
to be performed efficiently.
For example, it might be desirable to enumerate all
(possibly exponentially many) outcomes with polynomial delay.
It could also be desirable to check whether all outcomes agree
with a given partial ballot~$s$ (\emph{skeptical reasoning}).
For the sake of simplicity, in this paper
we will stick to the decision problem described above.
All tractability results that we obtain for the decision problem
can straightforwardly be extended
to tractability results for the above computational tasks.

For the judgment aggregation procedures~$F$ that we considered
above, \Outcome{F} is \ThetaP{2}-hard.
For an overview, see Table~\ref{table:bg-results}.

\begin{table}[!htb]
  \centering
  
  \begin{small}
  \begin{tabular}{@{} c |@{\quad}p{5cm} @{} } \toprule
       $F$ & complexity of \Outcome{F} \\
      \midrule
      \Kemeny & \ThetaP{2}-c \hfill \cite{LangSlavkovik14} \\[3pt]
      \ReversalScoring & \ThetaP{2}-c \hfill \cite{DeHaanSlavkovik17} \\[3pt]
      \Slater & \ThetaP{2}-c \hfill \cite{LangSlavkovik14} \\[3pt]
      \Young & \ThetaP{2}-c \hfill \cite{EndrissDeHaan15} \\[3pt]
      \MaxHamming & \ThetaP{2}-c \hfill \cite{DeHaanSlavkovik17} \\[3pt]
      \Tideman & \DeltaP{2}-c \hfill \cite{EndrissDeHaan15} \\
    \bottomrule
  \end{tabular}
  \end{small}
  
  \caption{The computational complexity of outcome determination
  for various procedures~$F$.}
  \label{table:bg-results}
  \vspace{-15pt}
\end{table}

\section{Krom and (Definite) Horn Formulas}

In this section, we consider the fragments
of Krom (2CNF), Horn and definite Horn formulas---%
for a formal definition of these fragments,
see the appendix.
These fragments can be used to express settings where
only basic dependencies between issues play a role---%
see Example~\ref{ex:basic-dependencies}
for an indication.

\begin{example}
\label{ex:basic-dependencies}
Krom (2CNF) formulas can be used to express
dependencies of the form ``if we decide to use
software tool~1 ($s_1$) or software tool~2 ($s_2$),
then we need to purchase the entire package ($p$):''
$(s_1 \vee s_2) \rightarrow p \equiv (\neg s_1 \vee p) \wedge (\neg s_2 \vee p)$.

Definite Horn formulas can be used to express
dependencies of the form ``if we hire both researcher~1 ($r_1$)
and researcher~2 ($r_2$), then we need to rent another office~$o$:''
$(r_1 \wedge r_2) \rightarrow o \equiv (\neg r_1 \vee \neg r_2 \vee o)$.
\end{example}

For some judgment aggregation rules these fragments
make computing outcomes tractable,
and for other judgment aggregation rules they do not.
We begin with considering the rules \Kemeny{} and \Slater{}.
Computing outcomes for these rules is tractable when restricted
to Krom formulas, but not when restricted to (definite) Horn formulas.
%

\begin{proposition*}
\label{prop:defhorn-kemeny}
\Outcome{\Kemeny} is \ThetaP{2}-hard even when
restricted to the case where~$\Gamma \in \DefHorn$.
\end{proposition*}


\begin{proposition*}
\label{prop:defhorn-slater}
\Outcome{\Slater} is \ThetaP{2}-hard even when
restricted to the case where~$\Gamma \in \DefHorn$.
\end{proposition*}

%

The following result refers to the notion of majority consistency
(see, e.g.,~\citealt{LangSlavkovik14}).
A profile~$\prof{r}$ is \emph{majority consistent}
(with respect to an integrity constraint~$\Gamma$)
if the majority outcome~$m_{\prof{r}}$ is consistent with~$\Gamma$.
A judgment aggregation procedure is
\emph{majority consistent} if for each integrity constraint~$\Gamma$ and
each profile~$\prof{r}$ that is majority consistent (w.r.t.~$\Gamma$),
the procedure outputs all and only those complete ballots
that agree with the (partial) ballot~$m_{\prof{r}}$.

\begin{theorem}
\label{thm:krom-majority-consistent}
For all judgment aggregation procedures~$F$
that are majority consistent, e.g.,~$F \in \SBs \Kemeny, \Slater \SEs$, 
\Outcome{F} is polynomial-time solvable
when~$\Gamma \in \Krom$.
\end{theorem}
\begin{proof}
The general idea behind this proof is 
to use the property that when~$\Gamma \in \Krom$,
the majority outcome~$m_{\prof{r}}$ is always $\Gamma$-consistent.
Let~$(\III,\Gamma,\prof{r},s)$ be an instance
of~\Outcome{F} with~$\Gamma \in \Krom$.
Let~$\prof{r} = (r_1,\dotsc,r_p)$.
We consider the majority outcome~$r^{*} = m_{\prof{r}}$.

We show that the partial ballot~$r^{*}$ is consistent with~$\Gamma$.
Suppose, to derive a contradiction, that~$r^{*}$ is inconsistent with~$\Gamma$.
Then there must be some clause~$(l_1 \vee l_2)$ of size~$2$
such that~$\Gamma \models (l_1 \vee l_2)$ and~$r^{*}$ sets both~$l_1$
and~$l_2$ to false.
By definition of~$r^{*}$, then a strict majority of the ballots in~$\prof{r}$ set~$l_1$
to false, and a strict majority of the ballots in~$\prof{r}$ set~$l_2$ to false.
By the pigeonhole principle then there must be some ballot~$r_i$ in~$\prof{r}$
that sets both~$l_1$ and~$l_2$ to false.
However, since~$\Gamma \models (l_1 \vee l_2)$, we get that~$r_i$ does
not satisfy~$\Gamma$, which is a contradiction with our assumption that all
ballots in the profile satisfy~$\Gamma$.
Thus, we can conclude that~$r^{*}$ is consistent with~$\Gamma$.

Since~$F$ is majority consistent, we know that~$F(\prof{r})$
contains all ballots that are consistent with both~$r^{*}$ and~$\Gamma$.
Since~$\Gamma \in \Krom$, deciding if~$F(\prof{r})$ contains
a ballot that is consistent with~$s$ can be done in polynomial time.
\end{proof}

We continue with the \MaxHamming{} procedure
for which computing outcomes is not tractable when restricted
to Krom formulas nor when restricted to definite Horn formulas.

\begin{proposition*}
\label{prop:outcome-maxhamming}
\Outcome{\MaxHamming} is \ThetaP{2}-hard even when restricted
to the case where~$\Gamma = \top$.
\end{proposition*}

\Outcome{\MaxHamming} restricted to the case where~$\Gamma = \top$
coincides with a problem known as \textsc{Closest String}
for binary alphabets
(see, e.g.,~\citealt{LiMaWang02}).
To the best of our knowledge, this is the first time that the exact
complexity of (this variant of) this problem has been identified.
\Outcome{\MaxHamming} is also very similar to the problem
of computing outcomes for the minimax rule in approval voting
\cite{BramsKilgourSanver04}.

\begin{corollary}
\label{cor:defhorn-krom-maxhamming}
\Outcome{\MaxHamming} is \ThetaP{2}-hard even when restricted
to the case where~$\Gamma \in \DefHorn \cap \Krom$.
\end{corollary}

Finally, we consider the procedure \Tideman{},
for which computing outcomes is tractable for both
Krom and Horn formulas.

\begin{theorem}
\label{thm:outcome-tideman-tractable}
Let~$\CCC$ be a class of propositional formulas
(or Boolean circuits) with
the following two properties:
\begin{itemize}
  \item $\CCC$ is closed under instantiation, i.e.,
    for any~$\Gamma \in \CCC$ and any partial truth
    assignment~$\alpha : \Var{\Gamma} \rightarrow \SBs 0,1 \SEs$
    it holds that~$\Gamma[\alpha] \in \CCC$; and
  \item satisfiability of formulas in~$\CCC$
    is polynomial-time solvable.
\end{itemize}
Then \Outcome{\Tideman} is polynomial-time
solvable when restricted to the case
where~$\Gamma \in \CCC$.
\end{theorem}
\begin{proof}[Proof (sketch)]
Let~$\CCC$ be a class of propositional formulas that satisfies the
conditions stated above, and let~$\Gamma \in \CCC$.
We can then compute~$\Outcome{\Tideman} = \SBs r^{*} \SEs$
by directly using the iterative definition of~$r^{*}$ given in the
description of the {ranked agenda procedure}.
This definition iteratively constructs partial
ballots~$s_0,\dotsc,s_{2n}$.
Ballot~$s_0$ is the empty ballot,
and for each~$i > 0$, ballot~$s_i$ is constructed from~$s_{i-1}$ by using only
the operations of instantiating the integrity constraint and
checking satisfiability of the resulting formula.
Due to the properties of~$\CCC$, these operations are all
polynomial-time solvable.
Thus, constructing~$r^{*} = s_{2n}$ can be done
in polynomial time.
\end{proof}

\begin{corollary}
\label{cor:outcome-tideman-tractable}
For each~$\CCC \in \SBs \Krom, \Horn \SEs$,
\Outcome{\Tideman} is polynomial-time
solvable when restricted to the case
where~$\Gamma \in \CCC$.
\end{corollary}

An overview of the complexity results
that we established in this section
can be found in Table~\ref{table:krom-horn-results}.

%
%
%

\begin{table}[!htb]
  \centering
  
  \begin{small}
  \begin{tabular}{@{} c |@{\quad}p{4.2cm} @{} } \toprule
       $F$ & complexity of \Outcome{F} \\
       & restricted to~$\Horn$ / $\DefHorn$ \\
      \midrule
      \Kemeny & \ThetaP{2}-c \hfill (Proposition~\ref{prop:defhorn-kemeny}) \\[3pt]
      \Slater & \ThetaP{2}-c \hfill (Proposition~\ref{prop:defhorn-slater}) \\[3pt]
      \MaxHamming & \ThetaP{2}-c \hfill (Corollary~\ref{cor:defhorn-krom-maxhamming}) \\[3pt] 
      \Tideman & in P \hfill (Corollary~\ref{cor:outcome-tideman-tractable}) \\ 
    \midrule
%
%
    \midrule
       $F$ & complexity of \Outcome{F} \\
       & restricted to~$\Krom$ \\
      \midrule
      \Kemeny & in P \hfill (Theorem~\ref{thm:krom-majority-consistent}) \\[3pt]
      \Slater & in P \hfill (Theorem~\ref{thm:krom-majority-consistent}) \\[3pt]
      \MaxHamming & \ThetaP{2}-c \hfill (Corollary~\ref{cor:defhorn-krom-maxhamming}) \\[3pt] 
      \Tideman & in P \hfill (Corollary~\ref{cor:outcome-tideman-tractable}) \\ 
    \bottomrule
  \end{tabular}
  \end{small}
  
  \caption{The computational complexity of outcome determination
  for several procedures~$F$ restricted to the case
  where~$\Gamma \in \Krom$,
  the case where~$\Gamma \in \Horn$,
  and the case where~$\Gamma \in \DefHorn$.}
  \label{table:krom-horn-results}
  \vspace{-10pt}
\end{table}

The results that we obtained for Horn formulas can all be
straightforwardly extended to the fragment of renamable Horn formulas---%
e.g., the fragment of renamable Horn formulas satisfies the
requirements of Theorem~\ref{thm:outcome-tideman-tractable}.
A propositional formula~$\varphi$ is \emph{renamable Horn}
if there is a set~$R \subseteq \Var{\varphi}$ of variables
such that~$\varphi$ becomes Horn when all literals over~$R$
are complemented.


\section{Boolean Circuits in DNNF}

Next, we consider the case where the integrity
constraints are restricted to Boolean circuits in Decomposable
Negation Normal Form (DNNF).
This is a class of Boolean circuits studied in the area
of knowledge compilation.
We illustrate how this class of circuits is useful
for judgment aggregation.

\subsection{Knowledge Compilation}

\emph{Knowledge compilation} (see, e.g.,~\citealt{DarwicheMarquis02}, \citealt{Darwiche14}, \citealt{Marquis15}) refers to a collection of approaches
for solving reasoning problems in the area of artificial intelligence and knowledge
representation and reasoning that are computationally intractable
in the worst-case asymptotic sense.
These reasoning problems typically involve knowledge in the form of
a Boolean function---often represented as a propositional formula.
The general idea behind these approaches is to split the reasoning process
into two phases: (1)~compiling the knowledge into a different
format that allows the reasoning problem to be solved efficiently,
and (2)~solving the reasoning problem using the compiled knowledge.
Since the entire reasoning problem is computationally intractable,
at least one of these two phases must be intractable.
Indeed, typically the first phase does not enjoy performance
guarantees on the running time---%
upper bounds on the size of the compiled knowledge
are often desired instead.
One of the advantages of this methodology is that one can reuse
the compiled knowledge for many instances,
which could lead to a smaller overall running time.

A prototypical example of a problem studied in the setting of knowledge
compilation is that of \emph{clause entailment}
(see, e.g.,~\citealt{DarwicheMarquis02},
~\citealt{CadoliDoniniLiberatoreSchaerf02}).
In this problem, one is given a knowledge base,
say in the form of a propositional formula~$\varphi$ in CNF,
together with a clause~$\delta$.
The question is to decide whether~$\varphi \models \delta$.
This problem is \co{\NP}-complete in general.
The knowledge compilation approach to solve this problem would be
to firstly compile the CNF formula~$\varphi$ into an equivalent
expression in a different format.
For example, one could consider the formalism of Boolean circuits
in \emph{Decomposable Negation Normal Form (DNNF)}
(or \emph{DNNF circuits}, for short).

DNNF circuits are a particular class of Boolean circuits in
\emph{Negation Normal Form (NNF)}.
A Boolean circuit~$C$ in NNF is a direct acyclic graph with a single root
(a node with no ingoing edges) where each leaf is labelled
with~$\top$,~$\bot$,~$x$ or~$\neg x$ for a propositional variable~$x$,
and where each internal node is labelled with~$\wedge$ or~$\vee$.
(An arc in the graph from~$N_1$ to~$N_2$ indicates that~$N_2$ is a child
node of~$N_1$.)
The set of propositional variables occurring in~$C$ is denoted
by~$\Var{C}$.
For any truth assignment~$\alpha : \Var{C} \rightarrow \SBs 0,1 \SEs$, we define
the truth value~$C[\alpha]$ assigned to~$C$ by~$\alpha$ in the usual
way, i.e., each node is assigned a truth value based on its label and
the truth value assigned to its children, and the truth value assigned to~$C$
is the truth value assigned to the root of the circuit.
DNNF circuits are Boolean circuits in NNF that satisfy the additional
property of decomposability.
A circuit~$C$ is \emph{decomposable} if for each conjunction in the circuit,
the conjuncts do not share variables.
That is, for each node~$d$ in~$C$ that is labelled with~$\wedge$ and for any
two children~$d_1,d_2$ of this node, it holds that~$\Var{C_1} \cap \Var{C_2} = \emptyset$,
where~$C_1,C_2$ are the subcircuits of~$C$ that have~$d_1,d_2$ as root,
respectively.
An example of a DNNF circuit is given
in Figure~\ref{fig:example-dnnf}.

\begin{figure}[htp]
  \begin{center}
    \begin{tikzpicture}[yscale=1.2]
      \node[] (x1) at (0,0) {$x_1$};
      \node[] (-x1) at (1.5,0) {$\neg x_1$};
      \node[] (x2) at (3,0) {$x_2$};
      \node[] (-x2) at (4.5,0) {$\neg x_2$};
      \node[] (and1) at (1.5,0.75) {$\wedge$};
      \node[] (and2) at (3,0.75) {$\wedge$};
      \node[] (or) at (2.25,1.125) {$\vee$};
      \draw[->] (x1) -- (and1);
      \draw[->] (x2) -- (and1);
      \draw[->] (-x1) -- (and2);
      \draw[->] (-x2) -- (and2);
      \draw[->] (and1) -- (or);
      \draw[->] (and2) -- (or);
    \end{tikzpicture}
  \end{center}
  \vspace{-5pt}
  \caption{An example of a DNNF circuit.}
  \label{fig:example-dnnf}
\end{figure}
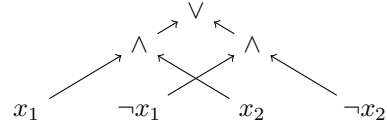

The problem of clause entailment can be solved in polynomial time
when the propositional knowledge is given as a DNNF circuit
\cite{DarwicheMarquis02}.
Moreover, every CNF formula can be translated to an equivalent
DNNF circuit---without guarantees on the size of the circuit.
Thus, one could solve the problem of clause entailment by
first compiling the CNF formula~$\varphi$ into an equivalent
DNNF circuit~$C$
(without guarantees on the running time or size of the result)
and then solving~$C \models \delta$
in time polynomial in~$|C|$.

Next, we will show how representation languages such as
DNNF circuits can be used in the setting of Judgment Aggregation,
and we will argue how Judgment Aggregation
can benefit from the approach of first compiling knowledge
(without performance guarantees) before using the compiled knowledge
to solve the initial problem.

\subsection{Algebraic Model Counting}

We will use the technique of algebraic model counting
\cite{KimmigVandenBroeckDeRaedt17}
to execute several judgment aggregation procedures efficiently
using the structure of DNNF circuits.
Algebraic model counting is a generalization of the problem
of counting models of a Boolean function
that uses the addition
and multiplication operators of a commutative semiring.

\begin{definition}[Commutative semiring]
A \emph{semiring} is a structure~$(\AAA,\oplus,\otimes,e^{\oplus},e^{\otimes})$,
where:
\begin{itemize}
  \item addition~$\oplus$ is an associative and commutative binary operation
    over the set~$\AAA$;
  \item multiplication~$\otimes$ is an associative binary operation
    over the set~$\AAA$;
  \item $\otimes$ distributes over~$\oplus$;
  \item $e^{\oplus} \in \AAA$ is the neutral element of~$\oplus$,
    i.e., for all~$a \in \AAA$,~$a \oplus e^{\oplus} = a$;
  \item $e^{\otimes} \in \AAA$ is the neutral element of~$\otimes$,
    i.e., for all~$a \in \AAA$,~$a \otimes e^{\otimes} = a$; and
  \item $e^{\oplus}$ is an annihilator for~$\otimes$,
    i.e., for all~$a \in \AAA$,~$e^{\oplus} \otimes a = a \otimes e^{\oplus} = e^{\oplus}$.
\end{itemize}
When~$\otimes$ is commutative, we say that the semiring is commutative.
When~$\oplus$ is idempotent, we say that the semiring is idempotent.
\end{definition}

\begin{definition}[Algebraic model counting]
Given:
\begin{itemize}
  \item a Boolean function~$f$ over a set~$\III$ of propositional variables;
  \item a commutative semiring~$(\AAA,\oplus,\otimes,e^{\oplus},e^{\otimes})$, and
  \item a labelling function~$\lambda : \Lit{\III} \rightarrow \AAA$ mapping literals
    over the variables in~$\III$ to values in the set~$\AAA$,
\end{itemize}
the task of \emph{algebraic model counting (AMC)} is to
compute:
\[ \bm{A}(f) =
  \bigoplus\limits_{\alpha : \III \rightarrow \SBs 0,1 \SEs \atop f(\alpha) = 1}
  \bigotimes\limits_{l \in \Lit{\III} \atop \lambda(l) = 1}
  \lambda(l). \]
\end{definition}

We can solve the task of algebraic model counting efficiently
for DNNF circuits when the semiring satisfies
an additional condition.

\begin{definition}[Neutral~$(\oplus,\alpha)$]
Let~$(\AAA,\oplus,\otimes,e^{\oplus},e^{\otimes})$ be a semiring,
and let~$\lambda : \Lit{\III} \rightarrow \AAA$ be a labelling function
for some set~$\III$ of propositional variables.
A pair~$(\oplus,\lambda)$ is \emph{neutral} if
for all~$x \in \III$ it holds that~$\lambda(x) \oplus \lambda(\neg x) = e^{\otimes}$.
\end{definition}

\begin{theorem}[{\citealt{KimmigVandenBroeckDeRaedt17}, Thm~5}]
\label{thm:amc}
When~$f$ is represented as a DNNF circuit,
and the semiring~$(\AAA,\oplus,\otimes,e^{\oplus},e^{\otimes})$
and the labelling function~$\lambda$ have the properties
that (i)~$\oplus$ is idempotent,
and (ii)~$(\oplus,\lambda)$ is neutral,
then the algebraic model counting problem is polynomial-time solvable---%
when given~$f$ and~$\lambda$ as input,
and when the operations of addition~($\oplus$) and multiplication~($\otimes$)
over~$\AAA$ can be performed in polynomial time.
\end{theorem}

We will use the result of Theorem~\ref{thm:amc} to show that
outcome determination for several judgment aggregation procedures
is tractable for the case where~$\Gamma$ is a DNNF circuit.
To do so, we will consider the following commutative, idempotent semiring
(also known as the \emph{max-plus algebra}).
We let~$\AAA = \ZZ \cup \SBs -\infty, \infty \SEs$,
we let $\oplus = \max$, $\otimes = +$, 
$e^{\oplus} = -\infty$, and~$e^{\otimes} = 1$.
Whenever we have a labelling function~$\alpha$ such
that~$(\oplus,\lambda)$ is neutral---%
i.e., such that~$\max(\lambda(x),\lambda(\neg x)) = 0$
for each~$x \in \III$---we satisfy the conditions
of Theorem~\ref{thm:amc}.

\begin{theorem}
\label{thm:dnnf-kemeny-slater}
\Outcome{\Kemeny} and \Outcome{\Slater}
are polynomial-time computable
when~$\Gamma$ is a DNNF circuit.
\end{theorem}
\begin{proof}
We prove the statement for \Outcome{\Kemeny}.
The case for \Outcome{\Slater} is analogous.
Let~$(\III,\Gamma,\prof{r},s)$ be an instance of \Outcome{\Kemeny}.
We solve the problem by reducing it to the problem of algebraic model
counting.
For~$(\AAA,\oplus,\otimes,e^{\oplus},e^{\otimes})$, we use the
max-plus algebra described above.
We construct the labelling function~$\lambda$ as follows.
For each~$x \in \III$, we count the number~$n_{x,1}$ of
ballots~$r \in \prof{r}$ such that~$r(x) = 1$
and we count the number~$n_{x,0}$ of ballots~$r \in \prof{r}$
such that~$r(x) = 0$.
That is, we let~$n_{x,0}$ and~$n_{x,1}$ be the majority strength
of~$\neg x$ and~$x$, respectively, in the profile~$\prof{r}$.
We pick a constant~$c_x$ such
that~$\max \SBs n'_{x,0}, n'_{x,1} \SEs = 0$
where~$n'_{x,0} = n_{x,0} + c_x$
and~$n'_{x,1} = n_{x,1} + c_x$.
We then let~$\lambda(x) = n'_{x,1}$ and~$\lambda(\neg x) = n'_{x,0}$.
This ensures that~$(\oplus,\lambda)$ satisfies the condition of
neutrality
(i.e., that~$\lambda(x) \oplus \lambda(\neg x) = e^{\otimes}$
for each~$x \in \III$).

This choice of~$(\AAA,\oplus,\otimes,e^{\oplus},e^{\otimes})$
and~$\lambda$ has the property that the
ballots~$r^{*} \in \Kemeny(\prof{r})$ are exactly
those complete ballots~$r^{*}$ that satisfy~$\Gamma$
and for which holds that~$\bm{A}(\Gamma) =
\bigotimes\nolimits_{l \in \Lit{\III}, r^{*}(l) = 1} \lambda(l)$.
That is, the set~$\Kemeny(\prof{r})$ consists of those
rational ballots that achieve the solution of the
algebraic model counting problem~$\bm{A}(\Gamma)$.
We can solve the instance of decision problem~\Outcome{\Kemeny}
by solving the algebraic model counting problem twice:
once for~$\Gamma$ and once for~$\Gamma[s]$.
The instance is a yes-instance
if and only if~$\bm{A}(\Gamma) = \bm{A}(\Gamma[s])$.
By Theorem~\ref{thm:amc},
this can be done in polynomial time.

To make this algorithm work for the case of \Outcome{\Slater},
one only needs to adapt the values of~$n_{x,0}$ and~$n_{x,1}$.
Instead of setting~$n_{x,0}$ and~$n_{x,1}$ to the majority
strength of~$\neg x$ and~$x$, respectively, we let~$n_{x,0} = 0$
if a strict majority of ballots~$r \in \prof{r}$ have that~$r(x) = 1$,
and we let~$n_{x,0} = 1$ otherwise.
Similarly, we let~$n_{x,1} = 0$ if a strict majority of
ballots~$r \in \prof{r}$ have that~$r(x) = 0$,
and we let~$n_{x,0} = 1$ otherwise.
\end{proof}

Representing the integrity constraint as a DNNF circuit makes
it possible to perform more tasks efficiently than just the
decision problem \Outcome{F}.
For example, the algorithms for algebraic model counting
can be used to produce a DNNF circuit that represents
the set~$F(\prof{r})$ of outcomes, allowing further operations
to be carried out efficiently.

\begin{theorem}
\label{thm:dnnf-reversal-scoring}
\Outcome{\ReversalScoring}
is polynomial-time computable
when~$\Gamma$ is a DNNF circuit.
\end{theorem}
\begin{proof}[Proof (sketch)]
The polynomial-time algorithm for \Outcome{\ReversalScoring}
is analogous to the algorithm described
for \Outcome{\Kemeny} described in the proof
of Theorem~\ref{thm:dnnf-kemeny-slater}.
The only modification that needs to be made to make this
algorithm work for \Outcome{\ReversalScoring} is
to adapt the numbers~$n_{x,0}$ and~$n_{x,1}$, for each~$x \in \III$.
Instead of identifying these numbers with the majority strength
of~$\neg x$ and~$x$, respectively,
we identify them with the total reversal score of~$x$ and~$\neg x$,
over the profile~$\prof{r}$.
That is, we let~$n_{x,0} = \sum\nolimits_{r \in \prof{r}} s_{\mtext{R}}(r,\neg x)$
and we let~$n_{x,1} = \sum\nolimits_{r \in \prof{r}} s_{\mtext{R}}(r,x)$.
For general propositional formulas~$\Gamma$, the reversal
scoring function~$s_{\mtext{R}}$ is NP-hard to compute.
However, since~$\Gamma$ is given as a DNNF circuit,
we can compute the scoring function~$s_{\mtext{R}}$,
and thereby~$n_{x,0}$ and~$n_{x,1}$, in polynomial time---%
by using another reduction to the problem of algebraic model counting.
We omit the details of this latter reduction.
\end{proof}

Intuitively, the results of Theorems~\ref{thm:dnnf-kemeny-slater}
and~\ref{thm:dnnf-reversal-scoring} are a consequence of the fact
that DNNF circuits allow polynomial-time weighted maximal model
computation, and that the judgment aggregation
procedures~\Kemeny{}, \Slater{} and \ReversalScoring{} are based
on weighted maximal model computation.
These results can therefore also straightforwardly be extended to
other judgment aggregation procedures that are based on
weighted maximal model computation.

\subsection{Other Results}

We can extend some previously established
results (Proposition~\ref{prop:outcome-maxhamming}
and Theorem~\ref{thm:outcome-tideman-tractable})
to the case of DNNF circuits.

\begin{corollary}
\label{cor:dnnf-tideman}
\Outcome{\Tideman}
is polynomial-time computable when restricted
to the case where~$\Gamma$ is a DNNF circuit.
\end{corollary}

\begin{corollary}
\label{cor:dnnf-maxhamming}
\Outcome{\MaxHamming}
is \ThetaP{2}-complete when restricted
to the case where~$\Gamma$ is a DNNF circuit.
\end{corollary}

A similar result for \Young{} follows from a result
that we will establish in the next section
(Proposition~\ref{prop:budget-young}).

\begin{corollary}
\label{cor:dnnf-young}
\Outcome{\Young}
is \ThetaP{2}-complete when restricted
to the case where~$\Gamma$ is a DNNF circuit.
\end{corollary}

An overview of the results established so far in this
section can be found in Table~\ref{table:dnnf-results}.

\begin{table}[!hbt]
  \centering
  
  \begin{small}
  \begin{tabular}{@{} c |@{\quad}p{4.0cm} @{} } \toprule
       $F$ & complexity of \Outcome{F} \\
      \midrule
      \Kemeny & in P \hfill (Theorem~\ref{thm:dnnf-kemeny-slater}) \\[3pt]
      \ReversalScoring & in P \hfill (Theorem~\ref{thm:dnnf-reversal-scoring}) \\[3pt]
      \Slater & in P \hfill (Theorem~\ref{thm:dnnf-kemeny-slater}) \\[3pt]
      \Young & \ThetaP{2}-c \hfill (Corollary~\ref{cor:dnnf-young}) \\[3pt] 
      \MaxHamming & \ThetaP{2}-c \hfill (Corollary~\ref{cor:dnnf-maxhamming}) \\[3pt]
      \Tideman & in P \hfill (Corollary~\ref{cor:dnnf-tideman}) \\ 
    \bottomrule
  \end{tabular}
  \end{small}
  
  \caption{The computational complexity of outcome determination
  for various procedures~$F$ restricted to the case where~$\Gamma$
  is a DNNF circuit.}
  \label{table:dnnf-results}
  \vspace{-10pt}
\end{table}

\subsection{A Compilation Approach}

The results of
Theorems~\ref{thm:dnnf-kemeny-slater}
and~\ref{thm:dnnf-reversal-scoring}
and Corollary~\ref{cor:dnnf-tideman} pave the way for
another approach towards finding cases where judgment aggregation
procedures can be performed efficiently.
The idea behind this approach is to compile the integrity constraint
into a DNNF circuit---regardless of whether this compilation process
enjoys a polynomial-time worst-case performance guarantee.
There are several off-the-shelf tools available that
compile CNF formulas into DNNF circuits using optimized methods
based on SAT solving algorithms
\cite{Darwiche04,MuiseMcIlraithBeckHsu12,OztokDarwiche14b}.
Since the class of DNNF circuits is expressively complete---i.e.,
every Boolean function can be expressed using a DNNF circuit---%
it is possible to compile any integrity constraint~$\Gamma$
into a DNNF circuit~$C_{\Gamma}$.

The downside is that
the circuit~$C_{\Gamma}$ could be of exponential size, or it could take
exponential time to compute it.
However, once the circuit~$C_{\Gamma}$ is computed and stored in
memory, one can use several judgment aggregation procedures
efficiently: \Kemeny{}, \Slater{}, \ReversalScoring{} and \Tideman{}.

Thus, this approach restricts the computational bottleneck to the
\emph{compilation phase}, before any judgments are solicited from the
individuals in the judgment aggregation scenario.
Once the compilation phase has been completed,
there are polynomial-time guarantees on the \emph{aggregation phase}
(polynomial in the size of the compiled DNNF circuit~$C_{\Gamma}$).

\subsection{CNF Formulas of Bounded Treewidth}

The tractability results for DNNF circuits can be leveraged to
get parameterized tractability results
for the case where the integrity constraint is a CNF formula
with a `treelike' structure.

\paragraph{Parameterized Complexity Theory \& Treewidth}

In order to explain the results that follow,
we briefly introduce some relevant concepts from
the theory of parameterized complexity.
For more details, we refer to textbooks on the topic
(see, e.g.,~\citealt{CyganEtAl15},
\citealt{DowneyFellows13}).
The central notion in parameterized complexity is
that of \emph{fixed-parameter tractability}---a notion of computational
tractability that is more lenient than the traditional
notion of polynomial-time solvability.
In parameterized complexity
running times are measured in terms of the input size~$n$
as well as a problem parameter~$k$.
Intuitively, the parameter is used to capture structure that is present
in the input and that can be exploited algorithmically.
The smaller the value of the problem parameter~$k$,
the more structure the input exhibits.
Formally, we consider \emph{parameterized problems} that
capture the computational task at hand as well as the choice
of parameter.
A parameterized problem~$Q$ is a subset of~$\Sigma^{*} \times \NN$
for some fixed alphabet~$\Sigma$.
An instance~$(x,k)$ of~$Q$ contains the problem input~$x \in \Sigma^{*}$
and the parameter value~$k \in \NN$.
A parameterized problem is \emph{fixed-parameter tractable}
there is a deterministic algorithm
that for each instance~$(x,k)$ decides
whether~$(x,k) \in Q$ and that runs in time~$f(k)\Card{x}^c$,
where~$f$ is a computable function of~$k$,
and~$c$ is a fixed constant.
Algorithms running within such time bounds are called \emph{fpt-algorithms}.
The idea behind these definitions is that fixed-parameter tractable running
times are scalable whenever the value of~$k$ is small.

A commonly used parameter is that of the treewidth of a graph.
Intuitively, the treewidth measures the extent to which a graph is
like a tree---trees and forests have treewidth~1,
cycles have treewidth~2, and so forth.
The notion of treewidth is defined as follows.
A \emph{tree decomposition} of a graph $G = (V,E)$ is a pair
$(\mathcal{T},(B_t)_{t \in T})$ where $\mathcal{T} = (T,F)$ is a tree
and $(B_t)_{t \in T}$ is a family of subsets of $V$ such that:
\begin{itemize}
  \item for every $v \in V$, the set $B^{-1}(v) = \SB t \in T \SM v \in B_t \SE$
    is nonempty and connected in $\mathcal{T}$; and
  \item for every edge $\SBs v,w \SEs \in E$, there is a $t \in T$ such that
    $v,w \in B_t$.
\end{itemize}
The \emph{width} of the decomposition $(\mathcal{T},(B_t)_{t \in T})$ is the
number $\max \SB \Card{B_t} \SM t \in T \SE - 1$.
The \emph{treewidth} of $G$ is the minimum of the widths
of all tree decompositions of $G$.
Let~$G$ be a graph and~$k$ a nonnegative integer.
There is an fpt-algorithm that computes
a tree decomposition of~$G$ of width~$k$ if it exists,
and fails otherwise \cite{Bodlaender96}.

\paragraph{Encoding Results}

We can then use results from the literature to establish tractability
results for computing outcomes of various judgment aggregation
procedures for integrity constraints whose variable interactions
have a treelike structure.
Let~$\Gamma = c_1 \wedge \dotsm \wedge c_m$ be a CNF formula.
The \emph{incidence graph of~$\Gamma$}
is the graph~$(V,E)$,
where~$V = \Var{\Gamma} \cup \SBs c_1,\dotsc,c_u \SEs$
and~$E = \SB \SBs c_j, x \SEs \SM 1 \leq j \leq m, x \in \Var{\Gamma},
\mtext{$x$ occurs in the clause~$c_j$} \SE$.
The \emph{incidence treewidth of~$\Gamma$} is defined as the
treewidth of the incidence graph of~$\Gamma$.

We can leverage the results of Theorems~\ref{thm:dnnf-kemeny-slater}
and~\ref{thm:dnnf-reversal-scoring}
and Corollary~\ref{cor:dnnf-tideman} to get fixed-parameter
tractability results for computing outcomes
of \Kemeny{}, \Slater{}, \ReversalScoring{} and \Tideman{}
for integrity constraints with small incidence treewidth.

\begin{proposition}[{\citealt{OztokDarwiche14},~%
\citealt{BovaCapelliMengelSlivovsky15}}]
Let~$\Gamma$ be a CNF formula of incidence treewidth~$k$.
Constructing a DNNF circuit~$\Gamma'$
that is equivalent to~$\Gamma$
can be done in fixed-parameter tractable time.
\end{proposition}

\begin{corollary}
The problems \Outcome{\Kemeny}, \Outcome{\Slater},
\Outcome{\ReversalScoring} and \Outcome{\Tideman}
are fixed-parameter tractable when parameterized by the
incidence treewidth of~$\Gamma$.
\end{corollary}

\section{Case Study: Budget Constraints}

In this section, we illustrate how the results of the previous section can
contribute to providing a computational complexity analysis for an
application setting.
The setting that we consider as an example is that of budget constraints.
This setting is closely related to that of \emph{Participatory Budgeting}
(see, e.g.,~\citealt{BenadeNathProcacciaShah17}), where citizens propose
projects and vote on which projects get funded by public money.
In the setting that we consider,
each issue~$x \in \III$ represents whether or not some measure
is implemented. Each such measure has an implementation cost~$c_x$ associated
with it.
Moreover, there is a total budget~$B$ that cannot be exceeded---that is,
each ballot (individual or collective) can set a set of variables~$x$ to true
such that the cumulative cost of these variables is at most~$B$
(and set the remaining variables to false).
The integrity constraint~$\Gamma$ encodes that the total budget~$B$
cannot be exceeded by the total cost of the variables that are set to true.
(For the sake of simplicity, we assume that all costs and the total budget
are all positive integers.)
%

The concepts and tools from judgment aggregation are
useful and relevant in this setting.
This is witnessed, for instance, by the fact that simply taking a majority vote
will not always lead to a suitable collective outcome.
Consider the example where there are three measures that are each associated
with cost~$1$, and where there is a budget of~$2$.
Moreover, suppose that there are three individuals. The first individual votes
to implement measures~$1$ and~$2$; the second votes for measures~$1$ and~$3$,
and the third for~$2$ and~$3$.
Each of the individuals' opinions is consistent with the budget.
However, taking a majority measure-by-measure vote results in implementing
all three issues, which exceeds the budget.
(In other words, the individual opinions~$r_1,r_2,r_3$ are all rational,
whereas the collective majority opinion~$m_{\prof{r}}$ is not.)
This example is illustrated
in Figure~\ref{figure:ja-budget-example}---%
in this figure, we encode the budget constraint using
a DNNF circuit~$\Gamma$.

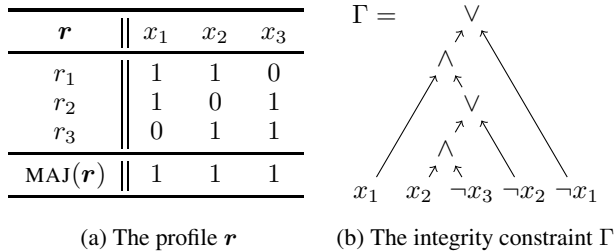
\begin{figure}[h!]
\centering
\begin{subfigure}{.5\linewidth}
  \centering
  \vspace{5pt}
  \begin{tabular}{c || c c c}
    \toprule
    $\prof{r}$ & $x_1$ & $x_2$ & $x_3$ \\
    \midrule
    $r_1$ & $1$ & $1$ & $0$ \\
    $r_2$ & $1$ & $0$ & $1$ \\
    $r_3$ & $0$ & $1$ & $1$ \\
    \midrule
    $\Majority(\prof{r})$ & $1$ & $1$ & $1$ \\
    \bottomrule
  \end{tabular}
  \vspace{5pt}
  \caption{The profile~$\prof{r}$}
  \label{figure:ja-budget-example-a}
\end{subfigure}%
\begin{subfigure}{0.5\linewidth}
\centering
\begin{tikzpicture}[xscale=0.7,yscale=0.8]
      \node[] (x1) at (0,0) {$x_1$};
      \node[] (-x1) at (4,0) {$\neg x_1$};
      \node[] (x2) at (1,0) {$x_2$};
      \node[] (-x2) at (3,0) {$\neg x_2$};
      \node[] (-x3) at (2,0) {$\neg x_3$};
      \node[] (and1) at (1.5,0.75) {$\wedge$};
      \node[] (or1) at (2,1.5) {$\vee$};
      \node[] (and2) at (1.5,2.25) {$\wedge$};
      \node[] (or2) at (2,3) {$\vee$};
      \draw[<-] (and1) -- (x2);
      \draw[<-] (and1) -- (-x3);
      \draw[<-] (or1) -- (and1);
      \draw[<-] (or1) -- (-x2);
      \draw[<-] (and2) -- (or1);
      \draw[<-] (and2) -- (x1);
      \draw[<-] (or2) -- (and2);
      \draw[<-] (or2) -- (-x1);
      \node[] at (0.15,3) {$\Gamma = $};
    \end{tikzpicture}
  \caption{The integrity constraint~$\Gamma$}
  \label{figure:ja-budget-example-b}
\end{subfigure}
  \caption{Example of an aggregation scenario
    with a budget constraint (for~$B = 2$
    and~$c_{x} = 1$ for all~$x \in \III$),
    where the budget constraint
    is represented as a DNNF circuit~$\Gamma$.}
  \label{figure:ja-budget-example}
  \vspace{-10pt}
\end{figure}

\subsection{Encoding into a Polynomial-Size DNNF Circuit}

To use the framework of judgment aggregation to model settings
with budget constraints, we need to encode budget constraints
using integrity constraints~$\Gamma$.
One can do this in several ways. We consider an encoding
using DNNF circuits (as in Figure~\ref{figure:ja-budget-example-b}).
Let~$\III$ be a set of issues, let~$\SBs c_x \SEs_{x \in \III}$ be
a vector of implementation costs, and let~$B \in \NN$ be a total budget.
We say that an integrity constraint~$\Gamma$ encodes
the budget constraint for~$\SBs c_x \SEs_{x \in \III}$ and~$B$
if for each complete ballot~$r : \III \rightarrow \SBs 0,1 \SEs$
it holds that~$r$ satisfies~$\Gamma$ if and only
if~$\sum\nolimits_{x \in \III, r(x) = 1} c_x \leq B$.

We can encode budget constraints efficiently using DNNF circuits
by expressing them as binary decision diagrams.
A \emph{binary decision diagram (BDD)} is a particular type of NNF circuit.
Let~$\Gamma$ be an NNF circuit. We say that a node~$N$ of~$\Gamma$
is a \emph{decision node} if (i)~it is a leaf or (ii)~it is
a disjunction node expressing~$(x \wedge \alpha) \vee (\neg x \wedge \beta)$,
where~$x \in \Var{\Gamma}$ and~$\alpha$ and~$\beta$
are decision nodes.
A binary decision diagram is an NNF circuit whose root is a decision
node.
A \emph{free binary decision diagram (FBDD)} is a BDD
that satisfies decomposability
(see, e.g.,~\citealt{DarwicheMarquis02},
\citealt{GergovMeinel94}).

\begin{theorem}
\label{thm:budget-dnnf-encoding}
For each~$\III$,~$\SBs c_x \SEs_{x \in \III}$ and~$B$,
we can construct a DNNF circuit~$\Gamma$ encoding the
budget constraint for~$\SBs c_x \SEs_{x \in \III}$ and~$B$
in time polynomial in~$B + \Card{\III}$.
\end{theorem}
\begin{proof}
We construct an FBDD~$\Gamma$ encoding the budget constraint
for~$\SBs c_x \SEs_{x \in \III}$ and~$B$ as follows.
Without loss of generality, suppose that~$c_x > 0$ for each~$x \in \III$.
Let~$\III = \SBs x_1,\dotsc,x_n \SEs$.
We introduce a decision node~$N_{i,j}$ for
each~$i \in \SBs 0,\dotsc,n \SEs$ and~$j \in \SBs 0,\dotsc,B \SEs$.
Take arbitrary~$i \in \SBs 0,\dotsc,n \SEs$ and~$j \in \SBs 0,\dotsc,B \SEs$.
If~$i = n$, we let~$N_{i,j} = \top$.
If~$i < n$, we distinguish two cases: either
(i)~$j' \leq B$ or (ii)~$j' > B$,
where~$j' = j + c_{x_i}$..
In case~(i), we let~$N_{i,j} = (x_i \wedge N_{i+1,j'}) \vee (\neg x_i \wedge N_{i+1,j})$.
In case~(ii), we let~$N_{i,j} = (x_i \wedge \bot) \vee (\neg x_i \wedge N_{i+1,j})$.
We let the root of the FBDD be the node~$N_{0,0}$---%
and we remove all nodes that are not descendants of~$N_{0,0}$.
Intuitively, the subcircuit rooted at~$N_{i,j}$ represents all truth assignments
to the variables~$x_{i+1},\dotsc,x_n$ that fit within a budget of~$B-j$.
For each node~$N_{i,j}$ it holds
that the variables in the leaves reachable from~$N_{i,j}$
are among~$x_{i+1},\dotsc,x_n$.
Therefore, we constructed an FBDD.
Moreover, each complete ballot~$r$ satisfies the circuit~$\Gamma$
if and only
if~$\sum\nolimits_{x \in \III, r(x) = 1} c_x \leq B$.
Thus,~$\Gamma$ is a DNNF circuit constructed in time
polynomial in~$B + \Card{\III}$
encoding the budget constraint
for~$\SBs c_x \SEs_{x \in \III}$ and~$B$.
\end{proof}

An example of a DNNF circuit resulting from the
encoding described in the proof of
Theorem~\ref{thm:budget-dnnf-encoding}---%
after some simplifications---%
can be found in Figure~\ref{figure:ja-budget-example-b}.

\subsection{Complexity Results}

Using the encoding result of
Theorem~\ref{thm:budget-dnnf-encoding},
we can establish polynomial-time solvability results
for computing outcomes for several judgment aggregation
procedures in the setting of budget constraints.

\begin{corollary}
\label{cor:budget-tractability}
\Outcome{\Kemeny}, \Outcome{\Slater},
\Outcome{\ReversalScoring}, 
and \Outcome{\Tideman}
are polynomial-time computable when restricted
to the case where~$\Gamma$ expresses a budget constraint.
\end{corollary}
\begin{proof}
The result follows from Theorems~\ref{thm:dnnf-kemeny-slater},
\ref{thm:dnnf-reversal-scoring} and~\ref{thm:budget-dnnf-encoding},
and Corollary~\ref{cor:dnnf-tideman}.
\end{proof}


For the \Young{} and \MaxHamming{} procedures,
we obtain intractability results for the case of budget constraints---%
for both procedures computing outcomes is \ThetaP{2}-hard.

\begin{proposition*}
\label{prop:budget-young}
\Outcome{\Young} is \ThetaP{2}-hard when restricted
to the case where~$\Gamma$ expresses a budget constraint.
\end{proposition*}

\begin{corollary}
\label{cor:budget-maxhamming}
\Outcome{\MaxHamming} is \ThetaP{2}-hard when restricted
to the case where~$\Gamma$ expresses
a budget constraint.
\end{corollary}
\begin{proof}
The result follows directly from
Proposition~\ref{prop:outcome-maxhamming}.
\end{proof}

An overview of the complexity results
that we established in this section
can be found in Table~\ref{table:budget-results}.

\begin{table}[!hbt]
  \centering
  
  \begin{small}
  \begin{tabular}{@{} c |@{\quad}p{4.0cm} @{} } \toprule
       $F$ & complexity of \Outcome{F} \\
      \midrule
      \Kemeny & in P \hfill (Corollary~\ref{cor:budget-tractability}) \\[3pt] 
      \ReversalScoring & in P \hfill (Corollary~\ref{cor:budget-tractability}) \\[3pt] 
      \Slater & in P \hfill (Corollary~\ref{cor:budget-tractability}) \\[3pt] 
      \Young & \ThetaP{2}-c \hfill (Proposition~\ref{prop:budget-young}) \\[3pt]
      \MaxHamming & \ThetaP{2}-c \hfill (Corollary~\ref{cor:budget-maxhamming}) \\[3pt]
      \Tideman & in P \hfill (Corollary~\ref{cor:budget-tractability}) \\
    \bottomrule
  \end{tabular}
  \end{small}
  
  \caption{The computational complexity of outcome determination
  for various procedures~$F$ restricted to the case where~$\Gamma$
  is a budget constraint.}
  \label{table:budget-results}
  \vspace{-10pt}
\end{table}

\section{Directions for Future Research}

In this paper, we provided a set of initial results
for restricted languages for judgment aggregation,
but these results are only the tip
of the iceberg that is to be explored.
We outline some directions for interesting future work
on this topic.

One first direction is to establish the complexity of \Outcome{F}
for cases that are left open in this paper---%
for example, for \Young{} and \ReversalScoring{} for
the case of Krom and (definite) Horn formulas.
Another direction is to pinpoint the complexity of \Outcome{F}
for the languages that we considered for other judgment aggregation
rules studied in the literature
(see, e.g.,~\citealt{LangPigozziSlavkovikVanderTorreVesic17}).

Yet another direction is to extend tractability results obtained in this
paper---e.g., for Krom and Horn formulas---to formulas that are `close'
to Krom or Horn formulas.
One could use the notion of backdoors for this
(see, e.g.,~\citealt{GaspersSzeider12}).

Finally, further restricted languages of propositional formulas or Boolean
circuits need to be studied, to get a more complete picture of where
the boundaries of the expressivity-tractability balance lie in the setting
of judgment aggregation.
A good source for additional languages is the field of knowledge
compilation (see, e.g.,~\citealt{DarwicheMarquis02}, \citealt{Darwiche14}, \citealt{Marquis15}),
where many restricted languages have been studied with
respect to their expressivity and support for performing various
operations tractably.

\section{Conclusion}

In this paper, we initiated the hunt for representation languages for
the setting of judgment aggregation that strike a balance between
(1)~allowing relevant computational tasks to be performed efficiently
and (2)~being expressive enough to model interesting and relevant
application settings.
Concretely, we considered Krom and (definite) Horn formulas,
and we studied the class of Boolean circuits in DNNF.
We studied the impact of these languages on the
complexity of computing outcomes for a number of
judgment aggregation procedures studied in the literature.
Additionally, we illustrated the use of these languages
for a specific application setting:
voting on how to spend a budget.

\appendix
\section{Appendix: Preliminaries}

We give an overview of some notions from
propositional logic and computational complexity
that we use in the paper.

\subsection{Propositional Logic}

\balance

Propositional formulas are constructed from propositional variables
using the Boolean operators~$\wedge,\vee,\rightarrow$, and~$\neg$.
A \emph{literal} is a propositional variable~$x$ (a \emph{positive literal})
or a negated variable~$\neg x$ (a \emph{negative literal}).
A \emph{clause} is a finite set of literals,
not containing a complementary pair~$x$,~$\neg x$,
and is interpreted as the disjunction of these literals.
A formula in \emph{conjunctive normal form (CNF)}
is a finite set of clauses, interpreted as the conjunction
of these clauses.
For each~$r \geq 1$, an \emph{$r$-clause} is a clause that contains
at most~$r$ literals, and $r\CNF{}$ denotes the class of all CNF
formulas consisting only of $r$-clauses.
$2\CNF$ is also denoted by~\Krom{},
and 2\CNF{} formulas are also known as \emph{Krom formulas}.
A \emph{Horn clause} is a clause that contains at most one
positive literal.
A \emph{definite Horn clause} is a clause that contains exactly
one positive literal.
We let \Horn{} denote the class of all CNF formulas that contain
only Horn clauses (\emph{Horn formulas}),
and we let \DefHorn{} denote the class of all CNF formulas that contain
only definite Horn clauses (\emph{definite Horn formulas}).

For a propositional formula~$\varphi$,~$\Var{\varphi}$ denotes
the set of all variables occurring in~$\varphi$.
Moreover, for a set~$X$ of variables,~$\Lit{X}$ denotes
the set of all literals over variables in~$X$,
i.e.,~$\Lit{X} = \SB x, \neg x \SM x \in X \SE$.
We use the standard notion of \emph{(truth)
assignments}~$\alpha : \Var{\varphi} \rightarrow \SBs 0,1 \SEs$
for Boolean formulas and \emph{truth} of a formula
under such an assignment.
For any formula~$\varphi$ and any truth assignment~$\alpha$,
we let~$\varphi[\alpha]$ denote the formula obtained from~$\varphi$
by instantiating variables~$s$ in the domain of~$\alpha$ with~$\alpha(x)$
and simplifying the formula accordingly.
By a slight abuse of notation,
if~$\alpha$ is defined on all~$\Var{\varphi}$,
we let~$\varphi[\alpha]$
denote the truth value of~$\varphi$ under~$\alpha$.

\subsection{Computational Complexity Theory}

We assume the reader to be familiar with the complexity classes
\P{} and \NP{}, and with basic notions such as polynomial-time
reductions.
For more details, we refer to textbooks on computational complexity
theory (see, e.g.,~\citealt{AroraBarak09}).

In this paper, we also refer to the complexity classes \ThetaP{2}
and \DeltaP{2}
that consist of all decision problems that can be solved by
a polynomial-time algorithm that queries an \NP{}
oracle~$O(\log n)$ or~$n^{O(1)}$ times, respectively.
\longversion{
Formally, algorithms with access to an oracle are defined as follows.
Let~$O$ be a decision problem.
A Turing machine~$\mathbb{M}$ with access to an \emph{$O$~oracle}
is a Turing machine with a dedicated \emph{oracle tape}
and dedicated states~$q_{\mtext{query}}$,~$q_{\mtext{yes}}$
and~$q_{\mtext{no}}$.
Whenever~$\mathbb{M}$ is in the state~$q_{\mtext{query}}$,
it does not proceed
according to the transition relation, but instead it transitions into
the state~$q_{\mtext{yes}}$ if the oracle tape contains
a string~$x$ that is a yes-instance for the problem~$O$, i.e.,~if~$x \in O$,
and it transitions into the state~$q_{\mtext{no}}$ if~$x \not\in O$.
Intuitively, the oracle solves arbitrary instances of~$O$
in a single time step.
The class \ThetaP{2} (resp.~\DeltaP{2}) consists of all decision problems~$Q$
for which there exists a deterministic Turing machine
that decides for each instance~$x$ of size~$n$
whether~$x \in Q$ in time polynomial in~$n$
by querying some oracle~$O \in \NP{}$
at most~$O(\log n)$ (resp.~$n^{O(1)}$) times.

Let~$\CCC$ be a class of propositional formulas.
The following problem
is complete for the class \ThetaP{2}
under polynomial-time reductions
when~$\CCC$ is the class of all propositional formulas
\cite{ChenToda95,Krentel88,Wagner90}.

\probdef{
  $\MaxModel(\CCC)$
  
  \emph{Instance:} A satisfiable propositional
    formula~$\varphi \in \CCC$, and a variable~$z \in \Var{\varphi}$.
  
  \emph{Question:} Is there a
    model of~$\varphi$ that sets a maximal number of
    variables in~$\Var{\varphi}$ to true (among all
    models of~$\varphi$) and that sets~$z$ to true?
}

For any class~$\CCC$ of propositional formulas,
we let~$\MaxModel(\CCC)$ denote the problem \MaxModel{}
restricted to formulas~$\varphi \in \CCC$.

%
%
%
%
}

\subsubsection{Acknowledgments.}
This work was supported by the Austrian Science Fund (FWF),
project~J4047.

\shortversion{\clearpage}

\DeclareRobustCommand{\DE}[3]{#3}

\bibliographystyle{aaai}

\longversion{
\clearpage
\appendix

\section{Additional Material: Lemmas and Proofs}

As additional material,
we provide proofs for all statements in the main
paper marked with a star ($\star$), as well
as additional lemmas used for these proofs.

\begin{lemma}
\label{lem:maxmodel:3cnf}
$\MaxModel(3\CNF)$ is \ThetaP{2}-complete.
\end{lemma}
\begin{proof}
We sketch a reduction from \MaxModel{} for arbitrary propositional formulas.
Let~$(\varphi,z)$ be an instance of \MaxModel{}.
By using the standard Tseitin transformation, 
we can transform~$\varphi$ into a 3\CNF{} formula~$\varphi'$
with~$\Var{\varphi'} = \Var{\varphi} \cup Z$ for some set~$Z$ of new
variables, such that for each truth assignment~$\alpha : \Var{\varphi}
\rightarrow \SBs 0,1 \SEs$ it holds that~$\varphi[\alpha]$ is true
if and only if there exists a truth assignment~$\beta : Z \rightarrow \SBs 0,1 \SEs$
such that~$\varphi'[\alpha \cup \beta]$ is true.

We then transform~$\varphi'$ into a 3\CNF{} formula~$\varphi''$
with~$\Var{\varphi''} = \Var{\varphi'} \cup Z'$,
for the set~$Z' = \SB z' \SM z \in Z \SE$ of fresh variables,
such that the maximal models of~$\varphi''$ correspond exactly to the
maximal models of~$\varphi$.
We define~$\varphi''$ as follows:
\[ \varphi'' = \varphi' \wedge \bigwedge\limits_{z \in Z}
  ((\neg z \vee \neg z') \wedge (z \vee z')). \]
Each model of~$\varphi''$ then must set the same number of variables
in~$Z \cup Z'$ to true---namely~$\Card{Z}$ of them.
\end{proof}

\begin{lemma}
\label{lem:maxmodel:defhorn-krom}
$\MaxModel(\Horn \cap \Krom)$ is \ThetaP{2}-com\-plete.
\end{lemma}
\begin{proof} 
We give a reduction from $\MaxModel(3\CNF)$.
Let~$(\varphi,z)$ be an instance of $\MaxModel(3\CNF)$,
where~$\Var{\varphi} = X = \SBs x_1,\dotsc,x_n \SEs$
and where~$\varphi$ consists of the clauses~$c_1,\dotsc,c_m$.
Without loss of generality, we may assume that each clause~$c_j$
is of size exactly~3.
Also, without loss of generality, we may assume that~$\varphi$
is satisfied by the ``all zeroes'' assignment, that is, by
the assignment~$\alpha_0$ such that~$\alpha_0(x_i) = 0$
for all~$i \in [n]$.
Moreover, we may assume without loss of generality that~$m \geq n$.
We construct an instance~$(\varphi',z')$
of $\MaxModel(\Horn \cap \Krom)$ as follows.

For each clause~$c_j$, we introduce fresh
variables~$y_j^u$ and~$y_{j,\ell}^u$,
for~$u \in [3]$ and~$\ell \in [n]$.
Moreover, for each~$x_i$, we introduce fresh
variables~$x_i^1$,~$x_i^0$,~$z_{i,\ell}^1$ for~$\ell \in [m+1]$
and~$z_{i,\ell}^0$ for~$\ell \in [m]$.
We then let~$\varphi'$ consist of the following clauses.
For each~$j \in [m]$, we add the clauses:
\[ (\neg y^1_j \vee \neg y^2_j), (\neg y^1_j \vee \neg y^3_j), (\neg y^2_j \vee \neg y^3_j), \]
ensuring that at most one variable among~$y^1_j,y^2_j,y^3_j$ can be true.
Moreover, for each~$j \in [m]$ and each~$u \in [3]$, we add the clauses:
\[ \begin{array}{r}
  (y^u_j \rightarrow y^u_{j,1}), (y^u_{j,1} \rightarrow y^u_{j,2}),\dotsc,
(y^u_{j,n-1} \rightarrow y^u_{j,n}), \\[3pt]
 (y^u_{j,n} \rightarrow y^u_j),
\end{array} \]
ensuring that the variables~$y^u_j$ and~$y^u_{j,\ell}$ get the same truth
value, for each~$u \in [3]$ and each~$j \in [m]$.

Then, for each~$i \in [n]$, we add the
clause~$(\neg x_i^1 \vee \neg x_i^0)$,
ensuring that at most one variable among~$x_i^1,x_i^0$ is true.
Moreover for each~$i \in [n]$ we add the clauses:
\[ \begin{array}{r}
 (x_i^1 \rightarrow z_{i,1}^1),(z_{i,1}^1 \rightarrow z_{i,2}^1),\dotsc,
(z_{i,m}^1 \rightarrow z_{i,m+1}^1), \\[3pt] (z_{i,m+1}^1 \rightarrow x_i^1),
\end{array} \]
and:
\[ \begin{array}{r}
  (x_i^0 \rightarrow z_{i,1}^0),(z_{i,1}^0 \rightarrow z_{i,2}^0),\dotsc,
(z_{i,m-1}^0 \rightarrow z_{i,m}^0), \\[3pt] (z_{i,m}^0 \rightarrow x_i^0),
\end{array} \]
ensuring that the variables~$x_i^u$ and~$z_{i,\ell}^u$ get the same truth
value, for each~$u \in \SBs 0,1 \SEs$ and each~$i \in [n]$.

Finally, we add the following clauses to~$\varphi'$, for each
clause~$c_j$ of~$\varphi$.
Let~$c_j$ be a clause of~$\varphi$,
and let~$l_{j,u}$ be the $u$-th literal in~$c_j$, for~$u \in [3]$.
If~$l_{j,u} = x_i$ for some~$i \in [n]$,
we add the
clause~$(y^u_j \rightarrow x_i^1)$,
and if~$l_{j,u} = \neg x_i$ for some~$i \in [n]$,
we add the
clause~$(y^u_j \rightarrow x_i^0)$.

To finish our construction,
we let~$z' = x_i^1$, for the unique~$i$ such that~$z = x_i$.

Before we show correctness of this reduction,
we establish several other properties of the formula~$\varphi'$.
Any maximal model of~$\varphi'$ sets at least~$n(m+1)+m(n+1)
= 2nm + n + m$ variables to true.
Since the ``all zeroes'' assignment~$\alpha_0$ satisfies~$\varphi$,
we can satisfy~$\varphi'$ by setting all variables~$x_i^0,z_{i,\ell}^0$
to true, setting all variables~$x_i^1,z_{i,\ell}^1$ to false,
and for each~$j \in [m]$ setting all variables~$y^u_j,y^u_{j,\ell}$
to true for some~$u \in [3]$, and setting all
variables~$y^{u'}_j,y^{u'}_{j,\ell}$ to false for the other~$u' \in [3]$.
This model of~$\varphi'$ sets~$2nm+n+m$ variables to true.

Moreover, by construction of~$\varphi'$, we know that each model
of~$\varphi'$ sets at most~$n(m+2)+m(n+1) = 2nm+2n+m$
variables to true.

By construction of~$\varphi'$, we know that any model of~$\varphi'$
sets variables~$x_i^u,z_{i,\ell}^u$ to true for \textbf{at most}
one~$u \in \SBs 0,1 \SEs$ for each~$i \in [n]$,
and that it sets variables~$y^u_j,y^u_{j,\ell}$ to true for \textbf{at most}
one~$u \in [3]$ for each~$j \in [m]$.
We argue that any maximal model of~$\varphi'$ must set
variables~$x_i^u,z_{i,\ell}^u$ to true for \textbf{exactly}
one~$u \in \SBs 0,1 \SEs$ for each~$i \in [n]$,
and must set variables~$y^u_j,y^u_{j,\ell}$ to true for \textbf{exactly}
one~$u \in [3]$ for each~$j \in [m]$.
Suppose that there is some maximal model of~$\varphi'$
that sets all variables~$x_i^1,x_i^0,z_{i,\ell}^1,z_{i,\ell}^0$
to false, for some~$i \in [n]$.
Then we know that this model can set at most~$2nm+2n-2$
variables to true.
Since~$m \geq n$, we know that this model cannot be maximal,
since there is a model that sets~$2nm+n+m > 2nm+2n-2$
variables to true.
From this we can conclude that each maximal model of~$\varphi'$
must set
variables~$x_i^u,z_{i,\ell}^u$ to true for exactly
one~$u \in \SBs 0,1 \SEs$ for each~$i \in [n]$.
An entirely similar argument can be used to show that
each maximal model of~$\varphi'$
must set variables~$y^u_j,y^u_{j,\ell}$ to true for exactly
one~$u \in [3]$ for each~$j \in [m]$.

Then, for each maximal model~$\alpha'$ of~$\varphi'$, we can construct
a truth assignment~$\alpha : X \rightarrow \SBs 0,1 \SEs$ as follows.
For each~$x_i \in X$, we let~$\alpha(x_i) = 1$ if and only if~$\alpha'$
sets~$x_i^1$ to true, and we let~$\alpha(x_i) = 0$ if and only
if~$\alpha'$ sets~$x_i^0$ to true.
Moreover, this truth assignment~$\alpha$ satisfies~$\varphi$.
To derive a contradiction, suppose that~$\alpha$ does not satisfy~$\varphi$,
that is, that there is some clause~$c_j$ of~$\varphi$ that~$\alpha$
does not satisfy.
Then there must be a clause of the form~$(y_j^u \rightarrow x_i^{u'})$
in~$\varphi'$, for some~$u \in [3]$ and some~$u' \in \SBs 0,1 \SEs$,
that is not satisfied by~$\alpha'$.
This is a contradiction with our assumption that~$\alpha'$
satisfies~$\varphi'$.
Therefore, we can conclude that~$\alpha$ satisfies~$\varphi$.

Conversely, for any model~$\alpha$ of~$\varphi$
we can construct a model~$\alpha'$ of~$\varphi'$ as follows.
For each~$i \in [n]$ and each~$u \in \SBs 0,1 \SEs$,%
~$\alpha'$ sets the variables~$x_i^u,z_{i,\ell}^u$ to true if and
only if~$\alpha(x_i) = u$.
Moreover, since~$\alpha$ satisfies~$\varphi$, we know that
for each~$j \in [m]$ there is some~$u_j \in [3]$ such
that~$\alpha$ satisfies the $u_j$-th literal in clause~$c_j$.
Then, for each~$j \in [m]$ and each~$u \in [3]$,%
~$\alpha'$ sets the variables~$y^u_j,y^u_{j,\ell}$
to true if and only if~$u = u_j$.
It is straightforward to verify that~$\alpha'$ satisfies~$\varphi'$.

We will now argue that there is a maximal model of~$\varphi$ that sets~$z$
to true if and only if there is a maximal model of~$\varphi'$ that sets~$z'$
to true.

$(\Rightarrow)$
Suppose that there is a maximal model~$\alpha$ of~$\varphi$ that sets~$z$ to true.
We can then construct a model~$\alpha'$ of~$\varphi'$, as described above.
It is easy to verify that~$\alpha'$ sets~$z'$ to true.
We argue that~$\alpha'$ is a maximal model of~$\varphi'$.
Suppose, to derive a contradiction, that~$\alpha'$ is not a maximal model of~$\varphi'$---%
that is, there is some model~$\beta'$ of~$\varphi'$ that sets more variables to true
than~$\alpha'$.
Then, as described above, we can construct a model~$\beta$ of~$\varphi$
from~$\beta'$.
It is straightforward to verify that~$\beta$ sets more variables in~$X$ to true than~$\alpha$.
This is a contradiction with our assumption that~$\alpha$ is a maximal model
of~$\varphi$.
Therefore, we can conclude that~$\alpha'$ is a maximal model of~$\varphi'$.

$(\Leftarrow)$
Conversely, suppose that there is a maximal model~$\alpha'$ of~$\varphi'$ that
sets~$z'$ to true.
We can then construct a model~$\alpha$ of~$\varphi'$, as described above.
It is easy to verify that~$\alpha$ sets~$z$ to true.
We argue that~$\alpha$ is a maximal model of~$\varphi$.
Suppose, to derive a contradiction, that~$\alpha$ is not a maximal model of~$\varphi$---%
that is, there is some model~$\beta$ of~$\varphi$ that sets more variables to true
than~$\alpha$.
Then, as described above, we can construct a model~$\beta'$ of~$\varphi'$
from~$\beta$.
It is straightforward to verify that~$\beta'$ sets more variables in~$\Var{\varphi'}$ to
true than~$\alpha'$.
This is a contradiction with our assumption that~$\alpha'$ is a maximal model
of~$\varphi'$.
Therefore, we can conclude that~$\alpha$ is a maximal model of~$\varphi$.
\end{proof}

\begin{lemma}
\label{lem:horn-kemeny}
\Outcome{\Kemeny} is \ThetaP{2}-hard even when
restricted to the case where~$\Gamma \in \Horn$.
\end{lemma}
\begin{proof} 
We give a reduction from~$\MaxModel(\Horn)$.
Let~$(\varphi,z)$ be an instance of~$\MaxModel(\Horn)$,
where~$\Var{\varphi} = X = \SBs x_1,\dotsc,x_n \SEs$.
We may assume without loss of generality that the ``all zeroes''
assignment~$\alpha_0 : X \rightarrow \SBs 0,1 \SEs$,
for which~$\alpha_0(x_i) = 0$ for all~$i \in [n]$, satisfies~$\varphi$.
We construct an instance~$(\III,\Gamma,\prof{r},s)$
of~\Outcome{\Kemeny}, with~$\Gamma \in \Horn$, as follows.

We let~$\III = X \cup \SB y_{i,j}, y'_{i,j} \SM i \in [n], j \in [3] \SE$.
We define~$\Gamma$ as follows:~$\Gamma =
\varphi \wedge \bigwedge\nolimits_{i \in [n]}
  ((y_{i,1} \wedge y_{i,2} \wedge y_{i,3} \rightarrow x_i) \wedge
  (y'_{i,1} \wedge y'_{i,2} \wedge y'_{i,3} \rightarrow x_i))$.
We define the profile~$\prof{r} = (r_1,r_2,r_3)$ as shown
in Figure~\ref{fig:profile-cbased-horn}.
\begin{figure}[b!]
\begin{center}
  \vspace{-5pt}
  \begin{tabular}{c || @{\ \ \ \ } c @{\ \ \ \ } | @{\ \ \ } c c c@{\ \ \ } | @{\ \ \ } c c c}
    \toprule
    $\prof{r}$ & $x_i$
      & $y_{i,1}$ & $y_{i,2}$ & $y_{i,3}$
      & $y'_{i,1}$ & $y'_{i,2}$ & $y'_{i,3}$ \\
    \midrule
    $r_1$ &$0$&$1$&$1$&$0$&$1$&$1$&$0$ \\
    $r_2$ &$0$&$1$&$0$&$1$&$1$&$0$&$1$ \\
    $r_3$ &$0$&$0$&$1$&$1$&$0$&$1$&$1$ \\
    \bottomrule
  \end{tabular}
\end{center}
\vspace{-10pt}
\caption{The profile~$\prof{r} = (r_1,r_2,r_3)$
in the proof of Lemma~\ref{lem:horn-kemeny}---%
here~$i$ ranges over~$[n]$.}
\label{fig:profile-cbased-horn}
\end{figure}
Finally, we let~$s$ be the partial ballot that only
sets~$z$ to~$1$.

Clearly, each rational ballot~$r^{*} \in \RRR(\III,\Gamma)$
must satisfy~$\varphi$, since~$\Gamma \models \varphi$.
Moreover, to satisfy~$\Gamma$, each rational ballot~$r^{*}$
must---for each~$i \in [n]$---either (i)~set~$x_i$ to~$1$
or (ii)~set at least one variable among~$y_{i,1},y_{i,2},y_{i,3}$
and at least one variable among~$y'_{i,1},y'_{i,2},y'_{i,3}$
to~$0$.
In case~(i), the total Hamming distance to the profile~$\prof{r}$
increases with~$3$,
and in case~(ii), the total Hamming distance to the profile~$\prof{r}$
increases with at least~$4$.
Therefore, the rational ballots~$r^{*}$ with minimal cumulative Hamming
distance to the profile~$\prof{r}$ correspond exactly to the
models of~$\varphi$ that set a maximal number of variables~$x \in X$
to true.
From this it immediately follows that
there exists some~$r^{*} \in \Kemeny(\prof{r})$
that agrees with~$s$ if and only if
there is a maximal model of~$\varphi$ that sets~$z$ to true.
\end{proof}

\begin{proof}[Proof of Proposition~\ref{prop:defhorn-kemeny} (sketch)]
We give a reduction from \Outcome{\Kemeny}
restricted to the case where~$\Gamma \in \Horn$.
Let~$(\III,\Gamma,\prof{r},s)$ be an instance
of~\Outcome{\Kemeny} with~$\Gamma \in \Horn$.
Let~$\prof{r} = (r_1,\dotsc,r_p)$.
Also, let~$c_1,\dotsc,c_m$ denote the clauses of~$\Gamma$.
Moreover, suppose that the clauses~$c_1,\dotsc,c_u$
are non-definite Horn clauses,
and that the clauses~$c_{u+1},\dotsc,c_m$ are definite Horn clauses.
We construct an equivalent
instance~$(\III',\Gamma',\prof{r}',s)$
of~\Outcome{\Kemeny} with~$\Gamma' \in \DefHorn$,
as follows.

Firstly, we let~$\III' = \III \cup \SB y_{j,\ell} \SM j \in [u], \ell \in [n+1] \SE$.
We obtain the definite Horn formula~$\Gamma'$ from~$\Gamma$
as follows.
Firstly, we add the clauses~$c_{u+1},\dotsc,c_m$ to~$\Gamma'$.
Then, for each non-definite Horn clause~$c_j$, with~$j \in [u]$,
we add a definite Horn clause~$(c_j \vee y_{j,\ell})$ to~$\Gamma'$
for each~$\ell \in [n+1]$.
We obtain the profile~$\prof{r}' = (r'_1,\dotsc,r'_p)$ from~$\prof{r}$
as follows.
For each~$i \in [p]$, we let~$r'_i$ agree with~$r_i$ on the issues
in~$\III$.
Moreover, for each~$i \in [p]$ and each~$x' \in \III' \backslash \III$,
we let~$r'_i(x') = 0$.
It is straightforward to verify that each~$r'_i$ is rational.

We firstly show that for each~$r^{*} \in \Kemeny(\prof{r}')$
and for each~$j \in [u]$, it holds that~$r^{*}$ sets all
variables~$y_{j,\ell}$ to~$0$.
We proceed indirectly, and suppose that this is not the case,
i.e., that there is some~$j \in [u]$ such that~$r^{*}$ does
not set all variables~$y_{j,\ell}$ to~$0$.
We distinguish two cases: either (i)~for all~$\ell \in [n+1]$
it holds that~$r^{*}$ sets~$y_{j,\ell}$ to~$1$,
or (ii)~this is not the case.
In case~(i), we know that the cumulative Hamming distance
from~$r^{*}$ to the profile~$\prof{r}'$ is at least~$p(n+1)$.
However, the ballot~$r_0$ such that~$r_0(x) = 0$ for all~$x \in \III'$
is rational and has cumulative distance of at most~$pn$
to~$\prof{r}'$.
Thus,~$r^{*}$ does not have minimal distance to~$\prof{r}'$,
which contradicts our assumption that~$r^{*} \in \Kemeny(\prof{r}')$.
In case~(ii), we know that there exists some~$\ell,\ell' \in [n+1]$
such that~$r^{*}$ sets~$y_{j,\ell}$ to~$1$
and~$y_{j,\ell'}$ to~$0$.
Then, we know that~$r^{*} \models c_j$,
since~$r^{*} \models (c_j \vee y_{j,\ell'})$.
However, then modifying~$r^{*}$ by setting~$y_{j,\ell}$ to~$0$
would result in a rational ballot with strictly smaller cumulative distance
to the profile~$\prof{r}'$,
which is a contradiction with our assumption that~$r^{*} \in \Kemeny(\prof{r}')$.
Thus, we can conclude that for each~$r^{*} \in \Kemeny(\prof{r}')$
and for each~$j \in [u]$, it holds that~$r^{*}$ sets all
variables~$y_{j,\ell}$ to~$0$.

It is then straightforward to verify that each~$r^{*} \in \Kemeny(\prof{r}')$
satisfies~$\Gamma$, and
that there exists a ballot~$r^* \in \Kemeny(\prof{r})$
such that~$s$ agrees with~$r^*$
if and only if there exists a ballot~$r^* \in \Kemeny(\prof{r}')$
such that~$s$ agrees with~$r^*$.
\end{proof}

\begin{lemma}
\label{lem:horn-slater}
\Outcome{\Slater} is \ThetaP{2}-hard even when
restricted to the case where~$\Gamma \in \Horn$.
\end{lemma}
\begin{proof} 
The proof of this statement is analogous to the proof
of Lemma~\ref{lem:horn-kemeny}---we use the same
reduction from~$\MaxModel(\Horn)$.
That is, we construct~$\III$,~$\Gamma$,~$\prof{r}$
and~$s$ in exactly the same way.
What remains to show is that this reduction
is also correct for the problem \Outcome{\Slater}.

Clearly, each rational ballot~$r^{*} \in \RRR(\III,\Gamma)$
must satisfy~$\varphi$, since~$\Gamma \models \varphi$.
Moreover, to satisfy~$\Gamma$, each rational ballot~$r^{*}$
must---for each~$i \in [n]$---either (i)~set~$x_i$ to~$1$
or (ii)~set at least one variable among~$y_{i,1},y_{i,2},y_{i,3}$
and at least one variable among~$y'_{i,1},y'_{i,2},y'_{i,3}$
to~$0$.
In case~(i), the total Hamming distance to the 
majority outcome~$m_{\prof{r}}$
increases with~$1$,
and in case~(ii), the total Hamming distance to the 
majority outcome~$m_{\prof{r}}$
increases with at least~$2$.
Therefore, the rational ballots~$r^{*}$ with minimal cumulative Hamming
distance to the profile~$\prof{r}$ correspond exactly to the
models of~$\varphi$ that set a maximal number of variables~$x \in X$
to true.
\end{proof}

\begin{proof}[Proof of Proposition~\ref{prop:defhorn-slater} (sketch)]
The proof of this statement is analogous to the proof
of Proposition~\ref{prop:defhorn-kemeny}.
That is, we take the reduction from \Outcome{\Kemeny}
to \Outcome{\Kemeny} from the proof of Proposition~\ref{prop:defhorn-kemeny},
and we employ it as a reduction from \Outcome{\Slater}
to \Outcome{\Slater}.
Since this reduction results in an instance where~$\Gamma \in \DefHorn$,
this suffices.
The argument for correctness of the reduction is entirely analogous.
\end{proof}

\begin{lemma}
\label{lem:max-hamming}
Let~$\varphi$ be a 3CNF formula with clauses~$c_1,\dotsc,c_m$
(all of size exactly~$3$)
and with~$n$ variables
such that~$\varphi \setminus \SBs c_1 \SEs$ is 1-in-3-satisfiable.
We can then in polynomial time
construct a set~$\III$ of issues together with a profile~$\prof{r}$
for~$\III$ (and for~$\Gamma = \top$), and positive integers~$u_1,u_2,u_3$
(with~$u_3 < u_1$)
that are polynomial in~$\Card{\varphi}$, and that depend
only on~$n$ and~$m$, such that:
\begin{itemize}
  \item if~$\varphi$ is 1-in-3-satisfiable, then the minimum max-Hamming distance from
    any ballot~$r^*$ to~$\prof{r}$ is~$u$, and moreover, there exists some ballot~$r^*$
    such that the Hamming distance from~$r^*$ to each individual ballot in~$\prof{r}$
    is exactly~$u_1$;
  \item if~$\varphi$ is not 1-in-3-satisfiable, then the minimum max-Hamming distance from
    any ballot~$r^*$ to~$\prof{r}$ is~$u_1+u_2$, and moreover, for each ballot~$r^*$
    that achieves his minimum max-Hamming distance to~$\prof{r}$ it holds
    that the Hamming distance from~$r^*$ to each individual ballot in~$\prof{r}$
    is exactly~$u+w$; and
  \item the Hamming distance from any ballot~$r^*$ that achieves the minimum
    max-Hamming distance to the profile~$\prof{r}$ to the all-zeroes ballot~$r_0$
    is exactly~$u_3$.
\end{itemize}
\end{lemma}
\begin{proof}
Take an arbitrary 3CNF formula~$\varphi$ with clauses~$c_1,\dotsc,c_m$
and~$\Var{\varphi} = \SBs x_1,\dotsc,x_n \SEs$ such that~$\varphi \setminus
\SBs c_1 \SEs$ is 1-in-3-satisfiable.
Without loss of generality, suppose that~$n$ is a power of~$2$.

We proceed in two steps.
In the first step, we will construct a set of issues and a set of ballots such that
the minimum max-Hamming distance to this set of ballots is lower than a particular
threshold if and only if~$\varphi$ is 1-in-3-satisfiable.
Then, in the second step, we will use these issues and ballots
to construct another set of issues and another set of ballots that satisfy the
conditions specified in the statement of the lemma.

We begin by introducing~$2n+4$ issues~$y_1,\dotsc,y_n,y'_1,\dotsc,y'_n, \allowbreak
z_1,\dotsc,z_4$,
together with a set of~$2 \log n + 3m$ ballots on these issues.
We define the first~$2 \log n$ ballots~$r_1,\dotsc,r_{\log n}$
and~$r'_1,\dotsc,r'_{\log n}$ as follows:
\[ r_i(y_j) = r_i(y'_j) = \begin{dcases*}
  1 & if the $i$-th bit of~$j$ is~$1$, \\
  0 & otherwise, \\
\end{dcases*} \]
and~$r'_i(y_j) = r'_i(y'_j) = 1-r_i(y_j)$.

It is straightforward to verify that any ballot~$r^*$ that sets exactly one
of~$y_j$ and~$y'_j$ to true (for each~$j \in [n]$) achieves the minimum
possible max-Hamming distance to these ballots (namely distance~$n$).
Moreover, any ballot that has a higher minimum max-Hamming distance to these
ballots has a Hamming distance strictly higher than~$n$ to more than one
of the ballots.
Intuitively, setting~$y_j$ to true in a ballot~$r^*$ corresponds to setting
variable~$x_j$ to true, and setting~$y'_j$ to true corresponds to setting
variable~$x_j$ to false.

Next, for each clause~$c_k$ of~$\varphi$,
we add a ballot~$s_k$, that is defined as follows.
For all variables~$x_j \not\in \Var{c_k}$, we let~$s_k(y_j) = s_k(y'_j) = 0$.
Then, for each literal~$l \in c_k$, if~$l = x_j$, we let~$s_k(y_j) = 1$
and~$s_k(y'_j) = 0$, and if~$l = \neg x_j$, we let~$s_k(y_j) = 0$
and~$s_k(y'_j) = 1$.
Moreover, for each~$\ell \in [4]$,~$s_k(z_\ell) = 0$.

Then, for each clause~$c_k$ of~$\varphi$,
we add ballots~$s'_k$ and~$s''_k$, that are defined as follows.
For all variables~$x_j \not\in \Var{c_k}$, we let~$s'_k(y_j) = s''_k(y_j) = s'_k(y'_j) = s''_k(y'_j) = 0$.
Then, for each literal~$l \in c_k$, if~$l = x_j$, we let~$s'_k(y_j) = s''_k(y_j) = 0$
and~$s'_k(y'_j) = s''_k(y'_j) = 1$, and if~$l = \neg x_j$, we let~$s'_k(y_j) = s''_k(y_j) = 1$
and~$s'_k(y'_j) = s''_k(y'_j) = 0$.
Moreover, for both~$\ell \in [2]$,~$s'_k(z_\ell) = 0$ and~$s''_k(z_\ell) = 1$,
and for both~$\ell \in [3,4]$,~$s'_k(z_\ell) = 1$ and~$s''_k(z_\ell) = 0$.

It is now routine to verify the following statements.
(1)~If~$\varphi$ is 1-in-3-satisfiable, then there is a ballot~$r^*$ that achieves
a minimum max-Hamming distance of~$n+4$ to these ballots---namely by setting
the variables~$z_\ell$ to~$0$, and by setting the variables~$y_j$ and~$y'_j$
according to the truth assignment witnessing exactly-1-satisfiability.
(2a)~If~$\varphi$ is not 1-in-3-satisfiable, then the minimum max-Hamming distance
of any ballot~$r^*$ to these ballots is strictly more than~$n+4$.
(2b)~If~$\varphi$ is not 1-in-3-satisfiable, then any ballot~$r^*$
has a cumulative Hamming distance to these ballots of at
least~$(2 \log n + 3m)(n+4)+2$.
(2c)~If~$\varphi$ is not 1-in-3-satisfiable, there is a ballot~$r^*$ that has
Hamming distance~$n+4$ to all ballots except one, to which it has Hamming
distance~$n+6$.
(3)~The Hamming distance from the all-zeroes ballot~$r_0$ to any ballot~$r^*$
achieving the minimum max-Hamming distance or the minimum cumulative Hamming
distance to these ballots is exactly~$n$.

Next, in the second step, we will use the issues and ballots that we constructed above
to construct the set~$\III$ of issues and the profile~$\prof{r}$ of ballots as specified
in the statement of the lemma.
We do this by making~$2 \log n + 3m$ copies of each of the issues~$y_j,y'_j,z_i$.
Then, we construct the profile~$\prof{r}$ that consists of~$2 \log n + 3m$ ballots,
each of which consists of a different ballot (among~$y_j,y'_j,z_i$) for each set of
copies of the issues.
This can be done as follows. Let~$t_1,\dotsc,t_b$ be the ballots that we defined above,
where~$b = 2 \log n + 3m$. Then let~$\prof{r} = \SBs t'_1,\dotsc,t'_b \SEs$.
For each~$i \in [2 \log n + 3m]$, the ballot~$t'_i \in \prof{r}$
agrees with ballot~$t_{i+\ell \mod 2 \log n + 3m}$ on the $\ell$-th copies of~$y_j,y'_j,z_i$,
for each~$\ell \in [2 \log n + 3m]$.

It is now straightforward to verify that if~$\varphi$ is 1-in-3-satisfiable,
then the minimum max-Hamming distance from any ballot~$r^*$ to~$\prof{r}$
is~$u_1 = (2 \log n + 3m)(n+4)$, and that there exists some ballot~$r^*$ that
has Hamming distance~$u$ to each ballot in~$\prof{r}$.
Also, if~$\varphi$ is not 1-in-3-satisfiable,
then the minimum max-Hamming distance from any ballot~$r^*$ to~$\prof{r}$
is~$u_1 + u_2$, where~$u_2 = 2(2 \log n + 3m)$,
and that any ballot~$r^*$ that achieves this minimum has Hamming distance
exactly~$u_1 + u_2$ to each ballot in~$\prof{r}$.

Moreover, any ballot~$r^*$ that achieves the minimum max-Hamming distance
to the profile~$\prof{r}$ has Hamming distance exactly~$u_3 = n(2 \log n + 3m)$
to the all-zeroes ballot~$r_0$.
\end{proof}

\begin{proof}[Proof of Proposition~\ref{prop:outcome-maxhamming}]
Membership in \ThetaP{2} (for the general case) has been shown before
\cite{DeHaanSlavkovik17}.
We show \ThetaP{2}-hardness for the case where~$\Gamma = \top$
by giving a reduction from the \ThetaP{2}-complete
problem of deciding whether the maximum number of variables set to true in
any satisfying assignment of a (satisfiable) propositional formula~$\varphi$ is odd.
Let~$\varphi$ be an arbitrary satisfiable propositional formula with~$n$
variables.
Suppose without loss of generality that~$n$ is even.

For each~$i \in [n]$, we construct a 3CNF formula~$\psi_i$ that is
1-in-3-satisfiable if and only if there is a truth assignment that satisfies~$\varphi$
and that sets at least~$i$ variables among~$\Var{\varphi}$ to true
(by NP-completeness of 1-in-3SAT, using the standard reduction).
We can do this in such a way that all of the formulas~$\psi_i$ have the same
number of clauses and the same number of variables,
and such that for each~$\psi_i$, there is some clause~$c \in \psi_i$ such
that~$\psi_i \setminus \SBs c \SEs$ is 1-in-3-satisfiable.

Then, since the formulas~$\psi_i$ satisfy the requirements for
Lemma~\ref{lem:max-hamming}, we can construct sets~$\III_1,\dotsc,\III_n$
of issues and profiles~$\prof{r}_1,\dotsc,\prof{r}_n$ such that for each~$i \in [n]$,
the issues~$\III_i$ and the profile~$\prof{r}_i$ satisfy the conditions mentioned
in the statement of Lemma~\ref{lem:max-hamming}.
We can do this in such a way that the sets~$\III_1,\dotsc,\III_n$ are disjoint.
Moreover, the profiles~$\prof{r}_1,\dotsc,\prof{r}_n$ have the same number~$b$
of individual ballots.
For each~$i \in [n]$, let~$\prof{r}_i$ consist of the ballots~$r^i_1,\dotsc,r^i_b$.
We then use these sets~$\III_1,\dotsc,\III_n$ and profiles~$\prof{r}_1,\dotsc,\prof{r}_n$
to construct a single set~$\III$ of issues and a single profile~$\prof{r}$.
We let~$\III = \bigcup\nolimits_{i = 1}^{n} \III_i \cup \SBs z \SEs$,
where~$z$ is a fresh propositional variable.
We let~$\prof{r}$ consist of the ballots~$r_1,\dotsc,r_b,r'_1,\dotsc,r'_b$,
that we will define below.

For each~$j \in [b]$, we define~$r_j$ as follows.
For each odd~$i \in [n]$ and each~$x \in \III_i$, we let~$r_j$ agree with~$r^i_j$,
i.e.,~$r_j(x) = r^i_j(x)$.
For each even~$i \in [n]$ and each~$x \in \III_i$, we let~$r_j(x) = 0$.
Finally, we let~$r_j(z) = 0$.

For each~$j \in [b]$, we define~$r'_j$ as follows.
For each even~$i \in [n]$ and each~$x \in \III_i$, we let~$r'_j$ agree with~$r^i_j$,
i.e.,~$r'_j(x) = r^i_j(x)$.
For each odd~$i \in [n]$ and each~$x \in \III_i$, we let~$r'_j(x) = 0$.
Finally, we let~$r_j(z) = 1$.

Finally, we let~$s$ be the partial ballot defined by letting~$l(z) = 1$
and~$l(x) = \star$ for all~$x \in \III \setminus \SBs z \SEs$.
We show that the maximum number of variables among~$\Var{\varphi}$
that are set to true in any satisfying assignment of~$\varphi$ is odd
if and only if there is some~$r^{*} \in \MaxHamming(\prof{r})$
that agrees with~$s$.

$(\Rightarrow)$
Suppose the maximum number of variables among~$\Var{\varphi}$
that are set to true in any satisfying assignment of~$\varphi$ is odd.
Then the number of formulas~$\psi_i$ that are not 1-in-3-satisfiable
is the same as the number of formulas~$\psi_i$ that are 1-in-3-satisfiable.
As a result, for any ballot~$r^{*}$ over~$\III \setminus \SBs z \SEs$
that minimizes the max-Hamming distance to~$\prof{r}$ (restricted
to~$\III \setminus \SBs z \SEs$), the Hamming distance to the
ballots~$r_1,\dotsc,r_b$ is equal to the Hamming distance to
the ballots~$r'_1,\dotsc,r'_b$.
As a result, any such ballot~$r^{*}$ over~$\III \setminus \SBs z \SEs$
minimizing the max-Hamming distance to~$\prof{r}$
(restricted to~$\III \setminus \SBs z \SEs$) can be extended to a
ballot minimizing the max-Hamming distance to~$\prof{r}$
by setting~$z$ to~$1$.

$(\Leftarrow)$
Suppose the maximum number of variables among~$\Var{\varphi}$
that are set to true in any satisfying assignment of~$\varphi$ is even.
Then there are more formulas~$\psi_i$ that are not 1-in-3-satisfiable
than formulas~$\psi_i$ that are 1-in-3-satisfiable.
As a result, for any ballot~$r^{*}$ over~$\III \setminus \SBs z \SEs$
that minimizes the max-Hamming distance to~$\prof{r}$ (restricted
to~$\III \setminus \SBs z \SEs$), the Hamming distance to the
ballots~$r_1,\dotsc,r_b$ is larger than the Hamming distance to
the ballots~$r'_1,\dotsc,r'_b$.
As a result, any ballot~$r^{*}$ that minimizes the max-Hamming distance
to~$\prof{r}$ must set~$z$ to~$0$.
\end{proof}

\probdef{
  \MinVC{}
  
  \emph{Instance:} A graph~$G = (V,E)$,
    and a vertex~$v^{*} \in V$.
  
  \emph{Question:} Is there a minimum-size vertex cover~$C \subseteq V$
    that includes~$v^{*}$?
}

\begin{lemma*}
\label{lem:maxvc}
\MinVC{} is \ThetaP{2}-complete.
\end{lemma*}
\begin{proof} 
Membership in \ThetaP{2} can be shown routinely.
We show \ThetaP{2}-hardness by reducing from the problem of deciding
whether the maximum number of variables satisfied by any model
for a given (satisfiable) propositional formula is odd.
Let~$\varphi$ be an arbitrary satisfiable propositional formula,
and let~$n = \Card{\Var{\varphi}}$.
Since the propositional satisfiability problem is NP-complete,
we can construct propositional formulas~$\psi_1,\dotsc,\psi_n$
such that for each~$i \in [n]$,~$\psi_i$ is satisfiable if and only
if~$\varphi$ can be satisfied by setting at least~$i$ variables among~$\Var{\varphi}$
to true.
Then, by NP-completeness of the problem of deciding whether a
graph has a clique of size at least~$m$,
we can transform these formulas~$\psi_i$ into graphs~$G_1,\dotsc,G_n$
together with positive integers~$m_1,\dotsc,m_n$ such that
for each~$i \in [n]$ it holds that~$G_i$ has a clique of size at least~$m_i$.
Moreover, we can ensure that no~$G_i$ has a clique of size~$m_i + 1$.
Let~$m_{\max} = \max \SBs m_1,\dotsc,m_n \SEs$.
Assume without loss of generality that~$m_{\max}$ is even.
Then we can straightforwardly transform the graphs~$G_1,\dotsc,G_n$ and the
integers~$m_1,\dotsc,m_n$ into graphs~$G'_1,\dotsc,G'_n$ and
integers~$m'_1,\dotsc,m'_n$ such that (1)~for each~$i \in [n]$
it holds that~$\psi_i$ is true if and only if~$G'_i$ has a clique of size~$m'_i$,
(2)~for each~$i \in [n]$,~$G'_i$ has no clique of size~$m'_i+1$,
and (3)~$m'_1 < m'_2 < \dotsm < m'_n$.
Assume without loss of generality that~$G'_1,\dotsc,G'_n$ are pairwise disjoint.
Then construct the graph~$G''$ by putting together~$G'_1,\dotsc,G'_n$,
adding two additional vertices~$v^{*}_1,v^{*}_2$,
connecting~$v^{*}_1$ to all vertices in~$G'_i$ for odd~$i$,
and connecting~$v^{*}_2$ to all vertices in~$G'_i$ for even~$i$.
It is straightforward to verify that every clique of~$G''$ of maximum size
does not contain~$v^{*}_2$ if and only if the maximum number of variables
satisfied by any model of~$\varphi$ is odd.
Then the instance~$(G''',v^{*}_2)$ of \MinVC{}---where~$G'''$ is the complement
of~$G''$---is a yes-instance of \MinVC{} if and only if the maximum number
of variables satisfied by any model of~$\varphi$ is odd.
This completes our proof of \ThetaP{2}-hardness.
\end{proof}

\begin{proof}[Proof of Proposition~\ref{prop:budget-young}]
We show \ThetaP{2}-hardness by reducing from \MinVC{}.
Let~$(G,v^{*})$ be an instance of \MinVC{}, where~$G = (V,E)$
with~$V = \SBs v_1,\dotsc,v_n \SEs$ and~$E = \SBs e_1,\dotsc,e_m \SEs$.
Without loss of generality, assume that~$2n+1$ is a multiple of~$3$.
Moreover, without loss of generality, assume that~$v^{*} = v_1$.

For each~$i \in [n]$, let~$d_i = \Card{\SB e_j \SM j \in [m], v_i \in e_j \SE}$
denote the degree of vertex~$v_i$.
For each~$j \in [m]$, let~$d'_j = \Card{\SB v_i \SM i \in [n], v_i \in e_j \SE}$
denote the degree of edge~$e_j$.
Moreover, for each~$i \in [n], j \in [m]$, let~$a_{i,j} = 1$ if and only
if~$v_i \in e_j$, and~$a_{i,j} = 0$ otherwise.
That is,~$a_{i,j}$ encodes whether~$v_i$ is incident to edge~$e_j$.
Also, for each~$i \in [n], j \in [m]$, let~$t_{i,j} = 1$ if and only if~$i \leq n+1-d'_j$,
and let~$t_{i,j} = 0$ otherwise.

We construct an instance~$(\III,\Gamma,\prof{r},s)$ of \Outcome{\Young} as follows.
We let~$\III = \SB x_j \SM j \in [m] \SE \cup \SBs y,z \SEs \cup \SB w_i,w'_i \SM i \in [e] \SE$.
Then, we let~$\Gamma$ be a budgetary constraint that assigns cost~$2$ to each
variable in~$\SB x_j \SM j \in [m] \SE$, cost~$2$ to~$z$, cost~$1$ to~$y$,
cost~$2m$ to each variable in~$\SB w_i,w'_i \SM i \in [e] \SE$,
and that assigns a total budget of~$12m+1$.
Then, we let~$\prof{r}$ be the profile as depicted in Figure~\ref{fig:outcome-young}.
It is straightforward to verify that each ballot in the profile satisfies the budgetary
constraint~$\Gamma$.
Finally, we let~$s$ be the partial ballot defined by~$l(y) = 0$ and~$l(v) = \star$
for all~$v \in \III \setminus \SBs y \SEs$.

\begin{figure*}[t]
\begin{center}
  \begin{tabular}{c || c c c c c | c c c c | c || c || c}
    $\prof{r}$ & $r_1$ & $r_2$ & $r_3$ & $\dotsm$ & $r_n$ & $r_{n+1}$ & $r_{n+2}$ & $\dotsm$ & $r_{2n}$ & $r_{2n+1}$ & $m_{\prof{r}}$ & cost \\
    \hline\hline
    $x_1$ & $a_{1,1}$ & $a_{2,1}$ & $a_{3,1}$ & $\dotsm$ & $a_{n,1}$ & $t_{1,1}$ & $t_{2,1}$ & $\dotsm$ & $t_{n,1}$ & 0 & 1 & $\bm{2}$ \\
    $x_2$ & $a_{1,2}$ & $a_{2,2}$ & $a_{3,2}$ & $\dotsm$ & $a_{n,2}$ & $t_{1,2}$ & $t_{2,2}$ & $\dotsm$ & $t_{n,2}$ & 0 & 1 & $\bm{2}$ \\
    $\vdots$ & $\vdots$ & $\vdots$ & $\vdots$ & $\ddots$ & $\vdots$ & $\vdots$ & $\vdots$ & $\ddots$ & $\vdots$ & $\vdots$ & $\vdots$ & $\vdots$ \\
    $x_m$ & $a_{1,m}$ & $a_{2,m}$ & $a_{3,m}$ & $\dotsm$ & $a_{n,m}$ & $t_{1,m}$ & $t_{2,m}$ & $\dotsm$ & $t_{n,m}$ & 0 & 1 & $\bm{2}$ \\
    \hline
    $y$ & $1$ & $0$ & $0$ & $\dotsm$ & 0 & $1$ & $1$ & $\dotsm$ & $1$ & $0$ & 1 & $\bm{1}$ \\
    \hline
    $z$ & $1$ & $1$ & $1$ & $\dotsm$ & $1$ & $0$ & $0$ & $\dotsm$ & $0$ & 0 & 0 & $\bm{2}$ \\
    \hline
    $w_1$ & $1$ & $1$ & $0$ & $1$ & \multicolumn{1}{c}{$1$} & $0$ & $\dotsm$ & $1$ & \multicolumn{1}{c}{$1$} & $0$ & 1 & $\bm{2m}$ \\
    $w_2$ & $1$ & $0$ & $1$ & $1$ & \multicolumn{1}{c}{$0$} & $1$ & $\dotsm$ & $1$ & \multicolumn{1}{c}{$0$} & $1$ & 1 & $\bm{2m}$ \\
    $w_3$ & $0$ & $1$ & $1$ & $0$ & \multicolumn{1}{c}{$1$} & $1$ & $\dotsm$ & $0$ & \multicolumn{1}{c}{$1$} & $1$ & 1 & $\bm{2m}$ \\
    \hline
    $w'_1$ & $1$ & $1$ & $0$ & $1$ & \multicolumn{1}{c}{$1$} & $0$ & $\dotsm$ & $1$ & \multicolumn{1}{c}{$1$} & $0$ & 1 & $\bm{2m}$ \\
    $w'_2$ & $1$ & $0$ & $1$ & $1$ & \multicolumn{1}{c}{$0$} & $1$ & $\dotsm$ & $1$ & \multicolumn{1}{c}{$0$} & $1$ & 1 & $\bm{2m}$ \\
    $w'_3$ & $0$ & $1$ & $1$ & $0$ & \multicolumn{1}{c}{$1$} & $1$ & $\dotsm$ & $0$ & \multicolumn{1}{c}{$1$} & $1$ & 1 & $\bm{2m}$ \\
  \end{tabular}
\end{center}
\caption{Construction of the profile~$\prof{r}$ in the proof of Proposition~\ref{prop:budget-young}.}
\label{fig:outcome-young}
\end{figure*}

Clearly, the majority outcome~$m_{\prof{r}}$ does not satisfy the budgetary constraint~$\Gamma$,
as all variables in~$\III \setminus \SBs z \SEs$ enjoy majority support, and the total cost
of these variables is~$14m+1 > 12m+1$.
There are two ways of saving a total cost of at least~$2m$ by deleting individual ballots:
either (1)~delete a set of ballots such that some~$w_i$ or~$w'_i$ is not supported by a majority
anymore, or (2)~delete a set of ballots such that all variables in~$\SB x_j \SM j \in [m] \SE \cup
\SBs z \SEs$ are not supported by a majority.
Option~(1) requires deleting more than~$1/3(2n+1)$ individual ballots,
as each~$w_i$ and~$w'_i$ enjoys a two-thirds majority support.
Without loss of generality, we may assume that the smallest vertex cover of~$G$
is of size less than~$1/6(2n+1)$---if this is not the case, we can simply add unconnected
vertices to increase~$n$.
We show that option~(2) requires less than~$1/3(2n+1)$ individual ballots.
Let~$C \subseteq V$ denote some vertex cover of~$G$.
Now remove from~$\prof{r}$ those individual ballots~$r_{i}$
and the~$\Card{C}$ individual ballots~$r_{n+1},\dotsc,r_{n+\Card{C}}$.
Without loss of generality, we may assume that these ballots~$r_{n+1},\dotsc,r_{n+\Card{C}}$
support all variables~$x_j$---again, if this were not the case, we could increase~$n$
by adding unconnected vertices.
It is straightforward to verify that removing these ballots results in a profile where
the variables in~$\SB x_j \SM j \in [m] \SE \cup \SBs z \SEs$ do not enjoy majority support.
Moreover, since~$\Card{C} < 1/6(2n+1)$, we deleted less than~$1/3(2n+1)$ individual ballots.
Thus, we can restrict our attention to deleting individual ballots that ensure that
the variables in~$\SB x_j \SM j \in [m] \SE \cup \SBs z \SEs$ do not enjoy majority support.

Let~$I \subseteq [2n+1]$ be a set of indices (of size smaller than~$1/3(2n+1)$)
such that if we delete the individual ballots~$r_{i}$ for all~$i \in I$, then
the variables in~$\SB x_j \SM j \in [m] \SE \cup \SBs z \SEs$ do not enjoy majority support.
By the way the Young judgment aggregation procedure is defined, it suffices to look
at sets~$I$ of even size.
Without loss of generality, we can assume that~$2n+1 \not\in I$---if this were not the case,
one could replace~$2n+1$ by any other index.
Moreover, without loss of generality, we can assume that for each~$i \in I \cap [n+1,2n]$
it holds that~$r_i$ accepts all variables in~$\SBs x_j \SM j \in [m] \SE$---if this were not the
case, we could replace such an~$i$ by another~$i' \in [n+1,2n]$ for which this is the case;
as mentioned above, since we can arbitrarily increase~$n$ by adding variables,
we may assume without loss of generality that enough such indices~$i'$ exist.
Now, let~$I_1 = I \cap [n]$ and let~$I_2 = I \cap [n+1,2n]$.
If~$\Card{I_2} > \Card{I_1}$, we know that in the resulting profile (after deleting the individual
ballots according to~$I$), the variable~$z$ has majority support.
This contradicts our assumption, and thus we can conclude that~$\Card{I_1} \geq \Card{I_2}$.
Then, if~$\Card{I_1} > \Card{I_2}$, we could replace some indices in~$I_1$ by other
indices in~$[n+1,2n] \setminus I_2$, and we would end up with another suitable
set~$I$ of indices.
Therefore, we can restrict our attention to the case where~$\Card{I_1} = \Card{I_2}$.

Each such set~$I$ corresponds to a vertex cover of~$G$ in the following way.
Let~$C_I \subseteq V$ be defined as~$C_I = \SB v_i \SM i \in I \cap [n] \SE$.
Suppose, to derive a contradiction, that~$C_I$ is not a vertex cover, i.e.,
that there is some~$e_j \in E$ such that~$C_I \cap e_j = \emptyset$.
Then in the profile resulting from deleting the individual ballots with indices
in~$I$, the variable~$x_j$ enjoys majority support.
This is a contradiction with our assumption that deleting the ballots corresponding
to~$I$ results in a profile where all variables
in~$\SB x_j \SM j \in [m] \SE \cup \SBs z \SEs$ do not enjoy majority support.
Thus, we can conclude that~$C_I$ is a vertex cover of~$G$.

We will now show that there is a minimum-size vertex cover~$C \subseteq V$
of~$G$ that includes~$v^{*}$ if and only if there is some~$r^{*} \in \Young(\prof{r})$
that agrees with~$s$.

$(\Rightarrow)$
Take a minimum-size vertex cover~$C \subseteq V$ of~$G$
that includes~$v^{*}$.
We show how to construct a minimum size set of individual ballots to delete
to result in a majority outcome~$r^{*}$ that satisfies~$\Gamma$.
Moreover, we show that deleting this set of ballots results in an outcome~$r^{*}$
that agrees with~$s$.
Define the set~$I$ of indices of ballots to delete as follows.
Let~$I = \SB i \in [n] \SM v_i \in C \SE \cup \SBs v_{n+1},\dotsc,v_{n+\Card{C}} \SEs$.
It is straightforward to verify, since~$C$ is a vertex cover of~$G$,
that deleting individual ballots according to~$I$ results in a consistent majority outcome
that does not include~$y$ (and thus that agrees with~$s$).
We show that~$I$ is of minimum size (among all such~$I$ that lead to a consistent
majority outcome).
Suppose, to derive a contradiction, that this is not the case, i.e., that there is some
suitable~$I'$ that is smaller than~$I$.
Then, as described above, we can construct a vertex cover~$C_{I'}$ of~$G$
that is smaller than~$C$, which is a contradiction.
Therefure,~$I$ is of minimum size.

$(\Leftarrow)$
Conversely, suppose that there is some~$r^{*} \in \Young(\prof{r})$
that agrees with~$s$, i.e., such that~$r^{*}(y) = 0$.
Then~$r^{*}$ results as the majority outcome of the profile after deleting
individual ballots according to some (minimum size) set~$I \subseteq [2n+1]$.
As described above, we can construct a vertex cover~$C_I$ of~$G$.
Since~$r^{*}(y) = 0$, it is straightforward to verify that~$v^{*} \in C$.
We show that~$C$ is a minimum size vertex cover.
Suppose, to derive a contradiction, that there exists a smaller vertex cover~$C'$
of~$G$.
Then define the set~$I'$ of indices of ballots to delete as follows.
Let~$I' = \SB i \in [n] \SM v_i \in C' \SE \cup \SBs v_{n+1},\dotsc,v_{n+\Card{C'}} \SEs$.
It is straightforward to verify, since~$C'$ is a vertex cover of~$G$,
that deleting individual ballots according to~$I'$ results in a consistent majority outcome.
Moreover, since~$C'$ is smaller than~$C$, we get that~$I'$ is smaller than~$I$.
This is a contradiction with our assumption that~$I$ is of minimum size.
Thus, we can conclude that~$C$ is a minimum size vertex cover of~$G$.
\end{proof}
}

\end{document}